\def\year{2022}\relax
\documentclass[letterpaper]{article} 
\usepackage{aaai22}  
\usepackage{times}  
\usepackage{helvet}  
\usepackage{courier}  
\usepackage[hyphens]{url}  
\usepackage{graphicx} 
\usepackage{natbib}  
\usepackage{caption} 
\DeclareCaptionStyle{ruled}{labelfont=normalfont,labelsep=colon,strut=off} 
\usepackage{algorithm}
\usepackage[noend]{algpseudocode}

\usepackage{newfloat}
\usepackage{listings}
\floatstyle{ruled}
\newfloat{listing}{tb}{lst}{}
\floatname{listing}{Listing}
\usepackage{amsthm}
\usepackage{mathtools}
\usepackage{xcolor}
\usepackage{booktabs}
\usepackage{tikz}
\usetikzlibrary{positioning}

\newcommand\n[1]{{\bar{#1}}}

\newcommand\eql[2]{{#1}\!=\!{#2}}
\newcommand{\To}{\!\!\to\!\!}

\newcommand{\adnan}[1]{{\color{red}#1}}

\newcommand\shrink[1]{}
\newcommand{\hl}[1]{{`#1'}}

\def\then{{\Rightarrow}}
\def\bthen{{\Leftarrow}}

\def\w{{\omega}}
\def\l{{\ell}}

\def\cd{{|}}

\newcommand{\iml}{\textsc{iml}}
\newcommand{\SR}{\textsc{sr}}
\newcommand{\SSR}{\textsc{ssr}}
\newcommand{\NR}{\textsc{nr}}
\newcommand{\SNR}{\textsc{snr}}
\newcommand{\CNF}{\textsc{cnf}}

\newcommand{\EXP}{\textsc{prune}}

\def\DE(#1,#2,#3){(#1,#2)\!\!\rightarrow\!\!#3}
\def\DE(#1,#2,#3){\xrightarrow[]{\footnotesize #1,\:#2}\!#3}

\newtheorem{corollary}{Corollary}
\newtheorem{lemma}{Lemma}
\newtheorem{definition}{Definition}
\newtheorem{proposition}{Proposition}

\title{On the Computation of Necessary and Sufficient Explanations}
\author{Paper ID: 5323}

\title{On the Computation of Necessary and Sufficient Explanations}

\author {
    Adnan Darwiche,
    Chunxi Ji
}
\affiliations {
    Computer Science Department\\ University of California, Los Angeles\\
    darwiche@cs.ucla.edu,  jich@cs.ucla.edu
}

\begin{document}

\maketitle

\begin{abstract}
The {\em complete reason} behind a decision is a Boolean formula
that characterizes why the decision was made.
This recently introduced notion has a number of applications, 
which include generating explanations, detecting decision bias
and evaluating counterfactual queries. 
Prime {\em implicants} of the complete reason are known as {\em sufficient reasons} for the decision and they correspond to what is known
as PI explanations and abductive explanations. 
In this paper, we refer to the prime {\em implicates} of a complete reason as {\em necessary reasons} for the decision. 
We justify this terminology semantically and show that necessary reasons correspond to what is known as contrastive explanations.
We also study the computation of complete reasons for multi-class decision trees and graphs with nominal 
and numeric features for which we derive efficient, closed-form complete reasons. 
We further investigate the computation of shortest necessary and sufficient reasons for a broad class of complete reasons,
which include the derived closed forms and the complete reasons for Sentential Decision Diagrams (SDDs).
We provide an algorithm which can enumerate their shortest necessary reasons in output polynomial time. 
Enumerating shortest sufficient reasons for this class of complete reasons is hard even for a single reason. 
For this problem, we provide an algorithm that appears to be quite efficient as we show empirically.
\end{abstract}

\section{Introduction}

Reasoning about the behavior of AI systems has been receiving significant attention recently,
particularly the decisions made by machine learning classifiers. Some methods operate
directly on classifiers, e.g.,~\cite{LIME,ANCHOR} while others operator on symbolic encodings of their 
input-output behavior, e.g.,~\cite{NarodytskaKRSW18,IgnatievNM19a} 
which may be compiled into tractable circuits~\cite{uai/ChanD03,ijcai/ShihCD18,aaai/ShihCD19,kr/ShiSDC20,kr/AudemardKM20,corr/abs-2107-01654}. 
When explaining decisions, the notion of a {\em sufficient reason} has been
well investigated. This is a minimal subset of an instance that is sufficient to trigger the decision 
and can therefore be used to explain why it was made. 
Sufficient reasons were introduced in~\cite{ijcai/ShihCD18} under 
the name of {\em PI explanations} and later referred to as {\em abductive explanations}~\cite{IgnatievNM19a}.\footnote{See, 
e.g.,~\cite{ijar/ChoiXD12,ANCHOR,ijcai/WangKB21} for some approaches that can be viewed as approximating sufficient reasons
and~\cite{JoaoApp} for a study of the quality of some of these approximations.}
Two related notions we discuss later are {\em contrastive explanations} as formalized in~\cite{aiia/IgnatievNA020} and 
{\em counterfactual explanations} as formalized in~\cite{kr/AudemardKM20}.\footnote{There is an
extensive body of work in philosophy, social science and AI that discusses contrastive explanations and counterfactual
explanations; see, e.g.,~\cite{Garfinkel1982-GARFOE-3,Lewis1986-LEWCE,temple1988contrast,Lipton1990-LIPCE,Counterfactual_blackbox,van2018contrastive,ai/Miller19,Wachter_Contrastive_2019,Counterfactual_Visual,verma2020counterfactual,Diverse_Counterfactual_Explanations}. 
While the definitions of these notions are sometimes variations or refinements on one another, they are not always compatible.}

\cite{ecai/DarwicheH20} introduced the {\em complete reason} for a decision as a Boolean formula that 
characterizes why a decision was made, and showed how it can be used to gather insights
about the decision. This includes generating explanations, determining
decision bias and evaluating counterfactual queries.
For example, it was shown that sufficient reasons correspond to the {\em prime implicants} of the complete reason.
Hence, if one has access to the complete reason behind a decision, then one can abstract
the computation of sufficient reasons away from the classifier and its encoding or compilation.
Consider a classifier for admitting applicants to an academic program based
five Boolean features~\cite{ecai/DarwicheH20}: passing the entrance exam (\(E\)),
being a first time applicant (\(F\)), having good grades (\(G\)), having work experience (\(W\)) and 
coming from a rich hometown (\(R\)). The positive instances of this classifier are
specified by the following Boolean formula:
\(
\Delta =
(e \vee g)\wedge(e \vee r)\wedge(e \vee w)\wedge(f \vee r)\wedge(\n{f} \vee g \vee w).
\)
Luna (\(\delta\)) passed the entrance exam, has good grades and work experience, 
comes from a rich hometown but is not a first time applicant (\(\delta = e,\n{f},g,r,w\)). 
The classifier will admit Luna. The complete reason for this decision is:
\(
\Gamma = (e\vee g)\wedge(e \vee w)\wedge(r)\wedge(\n{f}\vee g \vee w).
\)
There are four prime implicants of \(\Gamma\):  \(\{e,g,r\}, \{e,r,w\}, \{e,\n{f},r\}\) and \(\{g,r,w\}\).
Each is a minimal subset of instance \(\delta\) which is sufficient to trigger the admit decision. 
Even though the number of sufficient
reasons may be exponential,  the complete reason
can be compact and computed in linear time if the classifier is represented using a suitable form~\cite{ecai/DarwicheH20}.
Further insights can be obtained about a decision by analyzing its complete reason.
For example, the decision on Luna is {\em biased} as it would be different if she did not come
from a rich hometown. In that case, she would be denied admission {\em because} she does not come 
from a rich hometown and is not a first time applicant as this would be the only sufficient reason for rejection. 
These conclusions can be derived by operating directly, and efficiently, on the complete reason as shown in~\cite{ecai/DarwicheH20}.

More recently, \cite{darwiche2021quantifying} introduced the notion of {\em universal
literal quantification} to Boolean logic and used it to formulate complete reasons.
According to this formulation, we can obtain the above complete 
reason \(\Gamma\) by computing \(\forall e,\n{f},g,r,w \cdot \Delta\), to be explained later. 
We will base our treatment on this formulation while operating in a discrete instead of a Boolean setting. 
The conclusion section in~\cite{darwiche2021quantifying} proposed a generalization of universal literal 
quantification to discrete variables but without further discussion. We will adopt this definition, study it
further and exploit it to derive efficient, closed-form complete reasons
for multi-class decision trees and graphs with nominal (discrete) and numeric (continuous) features.
We will show that the obtained complete reasons belong to a particular logical form that arise
when explaining the decisions of a broader class of classifiers. We will further show that the {\em prime
implicates} of complete reasons correspond to contrastive explanations, which will provide
further insights into the semantics and utility of these explanations. We will refer to these prime implicates
as {\em necessary reasons} for the decision and semantically justify this terminology.
We will then propose an output polynomial algorithm for computing the shortest necessary reasons
of the identified class of complete reasons. We will finally show that computing shortest sufficient
reasons is hard for this class of complete reasons and propose an algorithm for computing them
which appears to be quite efficient based on an empirical evaluation. 
Proofs of all results can be found in the appendix.

\shrink{Proofs of all results can be found in~\cite{DarwicheJi22-arxiv}.}

\section{Syntax and Semantics of Discrete Formulas}
\label{sec:discrete}

We start by defining the syntax and semantics of discrete formulas which we use to capture classifiers with discrete features.
The treatment in this section is largely classical and provides obvious generalizations of what is known on Boolean logic. 
But we spell it out so we can provide a formal treatment of our 
upcoming results, especially that we sometimes depart from what may be customary. 

For a discrete variable \(X\) with values \(x_1, \ldots, x_n\), we will call \(\eql X x_i\) a {\em state} for variable \(X\).
A {\em discrete formula} is defined over a set of discrete variables as follows.
Every state or constant (\(\top\), \(\bot\)) is a discrete formula. 
If \(\alpha\) and \(\beta\) are discrete formulas, then \(\neg \alpha\), \(\alpha \vee \beta\) and \(\alpha \wedge \beta\) are discrete formulas.
A {\em positive literal} is a state \(\eql X x_i\) typically denoted by \(x_i\).
A {\em negative literal} is a negated state \(\neg (\eql X x_i),\) typically denoted by \(\n{x}_i\). 
A negative literal will also be called a state if the variable has only two values. 
A {\em clause} is a disjunction of literals with at most one literal per variable.
A {\em term} is a conjunction of literals with at most one literal per variable.
\shrink{
\footnote{This limits what can be 
represented by a clause or a term. For example, let \(S\) and \(\n{S}\) be a partition of states for variable \(X\). 
Then \(c = \bigvee_{x_i \in S} x_i\) is equivalent to \(t = \bigwedge_{x_i \in \n{S}} \n{x}_i\), where
each says that variable \(X\) has a state in set \(S\). Neither \(c\) is a clause nor \(t\) is a term
by our definition. However, \(c\) is a DNF and \(t\) is a CNF.}
}
A {\em CNF} is a conjunction of clauses.
A {\em DNF} is a disjunction of terms.
An {\em NNF} is defined as follows. Constants and literals are NNFs. If \(\alpha\) and \(\beta\) are NNFs, then
\(\alpha \vee \beta\) and \(\alpha \wedge \beta\) are NNFs (hence, conjunctions and disjunctions cannot be negated).
An NNF is {\em \(\vee\)-decomposable} iff for each disjunction \(\bigvee_i \alpha_i\) in the NNF, 
the disjuncts \(\alpha_i\) do not share variables.
An NNF is {\em \(\wedge\)-decomposable} iff for each conjunction \(\bigwedge_i \alpha_i\) in the NNF, 
the conjuncts  \(\alpha_i\) do not share variables.
An NNF is {\em positive} iff it contains only positive literals. 
Any NNF can be made positive by replacing negative literals \(\n{x}_i\) with \(\bigvee_{j \neq i} x_j\).
An NNF is {\em monotone} iff it is positive and does not contain distinct states \(x_i\) and \(x_j\) for any variable \(X\).

A {\em positive term} contains only positive literals (i.e., states).
The {\em conditioning} of discrete formula \(\Delta\) on positive term \(\gamma\) is denoted 
\(\Delta \cd \gamma\) and obtained as follows. For each state \(x_i \in \gamma\),  replace the
occurrences of \(x_i\) with \(\top\) and the occurrences of \(x_j\), \(j \neq i\), with \(\bot\). 
The formula \(\Delta \cd x_i\) does not mention variable \(X\). 
An {\em instance} is a positive term which contains precisely one state for each variable. 
If we condition a discrete formula on an instance, we get a Boolean formula that 
does not mention any variables (evaluates to true or false). 

The semantics of discrete formulas is symmetric to the semantics of Boolean formulas, except that the notion
of a {\em world} (truth assignment) is now defined as a function that maps each discrete variable to one of 
its states (a world corresponds to an instance). 
A world \(\w\) {\em satisfies} a discrete formula \(\alpha\), written \(\w \models \alpha\), precisely when \(\alpha \cd \w\) 
evaluates to true.
In this case, we say that world \(\w\) is a {\em model} of formula \(\alpha\).
Notions such as satisfiability, validity, implication and equivalence can now be defined for discrete formulas as in Boolean logic. 
For example, formula \(\alpha\) implies formula \(\beta\), written \(\alpha \models \beta\), iff every model of \(\alpha\) is
a model of \(\beta\). We next define the notions of implicants and implicates.
An {\em implicant} of a discrete formula \(\Delta\) is a term \(\delta\) such that \(\delta \models \Delta\).
The implicant is {\em prime} iff no other implicant \(\delta^\star\) is such that \(\delta^\star \subset \delta\).
An {\em implicate}  is a clause \(\delta\) such that \(\Delta \models \delta\).
The implicate is {\em prime} iff no other implicate \(\delta^\star\) is such that \(\delta^\star \subset \delta\).

\shrink{
Boolean logic falls as a special case when every discrete variable \(X\) has exactly 
two states: true and false. In Boolean logic, we use \(x\) (or \(X\)) to reference the true state and 
\(\n{x}\) (or \(\neg X\)) to reference the false state (which is kept implicit). For Boolean logic to fall properly 
as a special case of our discrete formulation, we have to interpret \(\n{x}\) (\(\neg X\)) as a second state 
of variable \(X\). As a result, every Boolean NNF is positive according to our earlier definitions.    
}

Our treatment will represent a classifier with discrete features and multiple classes \(c_1, \ldots, c_n\) by a set of mutually 
exclusive and exhaustive discrete formulas \(\Delta^1, \ldots, \Delta^n\), where the models of formula \(\Delta^i\) capture 
the instances in class \(c_i\). That is, instance \(\delta\) is in class \(c_i\) iff \(\delta \models \Delta^i\).
We refer to each \(\Delta^i\) as a {\em class formula.} When \(\delta \models \Delta^i\), we  say that instance \(\delta\) 
is {\em decided positively} by \(\Delta^i\). The complete reason for this decision will then be the formula
\(\forall \delta \cdot \Delta^i\). The next section will explain what \(\forall \delta\)  is and how to compute it efficiently.
In the upcoming discussion, we may use the engineering notation for Boolean operators when convenient,
writing \(x_1y_2 + x_2 z_3\), for example, instead of \((x_1 \wedge y_2) \vee (x_2 \wedge z_3)\).

\section{Quantifying States of Discrete Variables}
\label{sec:dquantify}

\cite{darwiche2021quantifying} introduced universal literal quantification for Boolean logic
and suggested the following generalization to discrete variables without further study.

\begin{definition} \label{def:d-quantify}
For formula \(\Delta\) and variable \(X\) with states \(x_1, \ldots, x_n\), the universal quantification
of state \(x_i\) from \(\Delta\) is defined as follows:
\(\forall x_i \cdot \Delta = (\Delta \cd x_i) \wedge \bigwedge_{j\neq i} (x_i \vee \Delta \cd x_j)\).
\shrink{
\begin{eqnarray*}
\forall x_i \cdot \Delta & = &  \Delta \cd x_i \wedge \bigwedge_{j\neq i} (x_i \vee \Delta \cd x_j) \\
\exists x_i \cdot \Delta & = &  \Delta \cd x_i \vee \bigvee_{j\neq i} (x_j \wedge \Delta \cd x_j)
\end{eqnarray*}
}
\end{definition}

Quantification is commutative so we can equivalently write \(\forall x \cdot (\forall y \cdot \Delta)\), 
\(\forall y \cdot (\forall x \cdot \Delta)\) or \(\forall \{x,y\} \cdot \Delta\). 
We will study Definition~\ref{def:d-quantify} and exploit it for computing complete reasons.

\begin{definition}\label{def:cr}
If instance \(\delta\) is decided positively by class formula \(\Delta\), then
\(\forall \delta \cdot \Delta\) is the \hl{complete reason} for the decision. 
\end{definition}

The next three results parallel Boolean ones in~\cite{darwiche2021quantifying}.
They are followed by two novel results.


\begin{proposition}\label{prop:d-quantify-b}
We have
\(\forall x_i \cdot \top = \top\) and \(\forall x_i \cdot \bot = \bot\);
\(\forall x_i \cdot x_i = x_i\) and \(\forall x_i \cdot \n{x}_i = \bot\);
\(\forall x_i \cdot x_j = \bot\) and \(\forall x_i \cdot \n{x}_j = x_i\) when \(j \neq i\);
\(\forall x_i \cdot y_j = y_j\) and \(\forall x_i \cdot \n{y}_j = \n{y}_j\) when \(X \neq Y\).
\end{proposition}

The next result shows when \(\forall x_i\) can be distributed.

\begin{proposition}\label{prop:distribute-and-or}
For discrete formulas \(\alpha\), \(\beta\) and state \(x_i\) of variable \(X\), we have
\(\forall x_i \cdot (\alpha \wedge \beta)  = (\forall x_i \cdot \alpha) \wedge (\forall x_i \cdot \beta)\).
Moreover, if variable \(X\) does not occur in both \(\alpha\) and \(\beta\), then
 \(\forall x_i \cdot (\alpha \vee \beta) = (\forall x_i \cdot \alpha) \vee (\forall x_i \cdot \beta)\).
\end{proposition}
 
Given Propositions~\ref{prop:d-quantify-b} and \ref{prop:distribute-and-or},
we can universally quantify states out of \(\vee\)-decomposable NNFs in linear time
while preserving \(\vee\)-decomposability in the resulting NNF.

\begin{proposition}\label{prop:d-quantify-compute}
Let \(\Delta\) be a \(\vee\)-decomposable NNF and \(\gamma\) be a set of states.
Then \(\forall \gamma \cdot \Delta\) can be obtained from \(\Delta\) as follows.
For each state \(x_i \in \gamma\), replace the occurrences of literals \(\n{x}_i\), \(x_j\) and \(\n{x}_j\), \(j \neq i\),
in \(\Delta\) with \(\bot\), \(\bot\) and \(x_i\), respectively.
\end{proposition}


Consider the class formula \(\Delta = \n{x}_1(x_2+\n{y}_1)(\n{y}_1+z_1)\) over ternary variables \(X\), \(Y\), \(Z\)
and instance \(\delta = x_2, y_2, z_1\) which is decided positively by \(\Delta\). The complete
reason for this decision is \(\forall \delta \cdot \Delta\). Since \(\Delta\) is 
\(\vee\)-decomposable, Proposition~\ref{prop:d-quantify-compute} gives
\(
\forall x_2, y_2, z_1 \cdot \Delta = (x_2)(x_2+y_2)(y_2+z_1) = x_2(y_2+z_1).
\)
Hence, this instance was decided positively because it has characteristic \(x_2\) and one of  
the characteristics \(y_2\) and \(z_1\).

We next identify conditions that allow the distribution of \(\forall x_i\) over disjuncts
that share variables.

\begin{proposition}\label{prop:distribute-or-ind}
Consider positive NNFs \(\alpha,\) \(\beta\) and state \(x_i\) of variable \(X\).
If \(x_i\) does not occur in \(\alpha, \beta\), or 
\(x_j\) does not occur in \(\alpha, \beta\) for all \(j \neq i\), 
then \(\forall x_i \cdot (\alpha \vee \beta) = (\forall x_i \cdot \alpha) \vee (\forall x_i \cdot \beta)\).
\end{proposition}

For a Boolean variable \(X\) with states \(x\) and \(\n{x}\), 
Proposition~\ref{prop:distribute-or-ind} says that we can distribute \(\forall x\)
over disjuncts \(\alpha\) and \(\beta\) even if they mention 
literal \(\n{x}\) (but do not mention \(x\)).
This is a novel result compared to~\cite{darwiche2021quantifying}.

Next is another novel condition that licenses the distribution of \(\forall x_i\) over disjuncts, 
which we use to derive closed forms for the complete reasons of decision trees and graphs.
\shrink{
\begin{proposition}\label{prop:wtop}
For NNF \(\alpha\) and variable \(X\), we have
\[\forall x_i \cdot (\alpha \vee \bigvee_{x_k \in S} x_k) = (\forall x_i \cdot \alpha) \vee (\forall x_i \cdot \bigvee_{x_k \in S} x_k)\]
if \(X\) occurs in \(\alpha\) only in disjunctions of the form \(\bigvee_{x_k \in S'} x_k\) where \(S' \supseteq S\).
\end{proposition}
}
\begin{proposition}\label{prop:wtop}
Let \(\alpha\) be an NNF, \(S\) be a set of states for variable \(X\) and \(\beta = \bigvee_{x_k \in S} x_k\).
If variable \(X\) occurs in \(\alpha\) only in disjunctions of the form \(\bigvee_{x_k \in S'} x_k\) where \(S' \supseteq S\)
are states of variable \(X\), then
\(\forall x_i \cdot (\alpha \vee \beta) = (\forall x_i \cdot \alpha) \vee (\forall x_i \cdot \beta)\).
\end{proposition}

Consider variables \(X\) (\(x_1, \ldots, x_4\)) and \(Y\) (\(y_1, y_2\))
and the formulas \(\alpha = y_1 (x_1 + x_2 + x_4)\) and \(\beta = x_1 + x_2\). 
We can invoke Proposition~\ref{prop:wtop} to distribute \(\forall x_1\) using
\(S = \{x_1,x_2\}\) and \(S' = \{x_1,x_2,x_4\}\).
Hence, \(\forall x_1 \cdot (\alpha \vee \beta)
= \forall x_1 \cdot (y_1 (x_1 + x_2 + x_4)) \vee  \forall x_1 \cdot (x_1 + x_2)
= y_1 x_1 + x_1 = x_1\).
Propositions~\ref{prop:distribute-and-or} and~\ref{prop:distribute-or-ind} 
do not license this distribution of \(\forall x_i\) though.

\shrink{
\adnan{The following analysis is not complete as it does not over \(\forall \n{x}\).}
For binary formulas, Proposition~\ref{prop:wtop} reduces to the following.
If \(\alpha\) is independent of \(\n{x}\), 
then \(\forall x \cdot (x \vee \alpha) = (\forall x \cdot x) \vee (\forall x \cdot \alpha) = x \vee  (\forall x \cdot \alpha) = x \vee \alpha\).
Moreover, if \(\alpha\) is independent of \(x\), 
then \(\forall x \cdot (\n{x} \vee \alpha) = (\forall x \cdot \n{x}) \vee (\forall x \cdot \alpha) = \forall x \cdot \alpha = \alpha \cd x\).
Note that  
\(\forall x \cdot (x \vee \alpha) = x \vee \alpha\) and 
\(\forall x \cdot (\n{x} \vee \alpha) = \alpha \cd x\) even when \(\alpha\) is not independent of \(x\) and \(\alpha\) is not independent of \(\n{x}\).
However, to obtain these equalities using distribution (as opposed to Definition~\ref{def:d-quantify}), we need the
stated conditions. Another way to derive these equalities is to observe that \(x \vee \alpha = x \vee (\alpha \cd \n{x})\)
and \(\n{x} \vee \alpha = \n{x} \vee (\alpha \cd x)\). We can then use distribution without the stated conditions
since \(X\) does not occur in \(\alpha \cd \n{x}\) or \(\alpha \cd x\).
}

\section{The Complete Reasons for Decision Graphs}
\label{sec:cr-cf}

 \begin{figure}[tb]
        \centering
        \scalebox{0.80}{
        \begin{tikzpicture}[
        roundnode/.style={circle ,draw=black, thick},
        squarednode/.style={rectangle, draw=black, thick},
        ]
        \node[squarednode]      (SAT)                              {\Large SAT};
        \node[squarednode]        (GPA)       [below=of SAT, xshift = -1cm] {\Large GPA};
        \node[squarednode]        (Essay)       [below=of SAT, xshift = 1cm, yshift = -1.5cm] {\Large Essay};
        \node[squarednode]        (Interview)       [below=of GPA, yshift = -2cm] {\Large Interview};
        \node[roundnode, inner sep=1pt, align=center, text width = 13mm]        (Reject1)   [below=of GPA, xshift = -2.5cm, yshift = -1cm] {\Large Reject};
        \node[roundnode, inner sep=1pt, align=center, text width = 13mm]        (Reject2)   [below=of Interview, xshift = -2cm, yshift = -0.5cm] {\Large Reject};
        \node[roundnode, inner sep=1pt, align=center, text width = 13mm]        (Accept)   [below=of Interview, xshift = 3cm, yshift = -0.5cm] {\Large Accept};
        
        \draw[-latex, thick] (SAT.240) -- node [anchor = center, xshift = -8mm, yshift = 1mm] {$<1450$} (GPA.north);
        \draw[-latex, thick] (SAT.300) -- node [anchor = center, xshift = 7mm, yshift = 1mm] {$\geq 1450$} (Essay.north);
        \draw[-latex, thick] (GPA.300) --  node [anchor = center, xshift = 4mm, yshift = 1.5mm] {high} (Essay.180);
        \draw[-latex, thick] (GPA.south) -- node [anchor = center, xshift = -6.5mm, yshift = 1mm] {medium} (Interview.north);
        \draw[-latex, thick] (GPA.240) --  node [anchor = center, xshift = -5mm, yshift = 2mm] {low} (Reject1.north);
        \draw[-latex, thick] (Essay.240) -- node [anchor = center, xshift = -4mm, yshift = 1mm] {fail} (Interview.45);
        \draw[-latex, thick] (Essay.270) -- node [anchor = center, xshift = 5mm] {pass} (Accept.80);
        \draw[-latex, thick] (Interview.240) -- node [anchor = center, xshift = -7mm] {fail} (Reject2.north);
        \draw[-latex, thick] (Interview.300) -- node [anchor = center, xshift = -9mm] {pass} (Accept.100);
        \end{tikzpicture}
        }
        \caption{A classifier in the form of a decision graph. \label{fig:admissions-dg}}
    \end{figure}
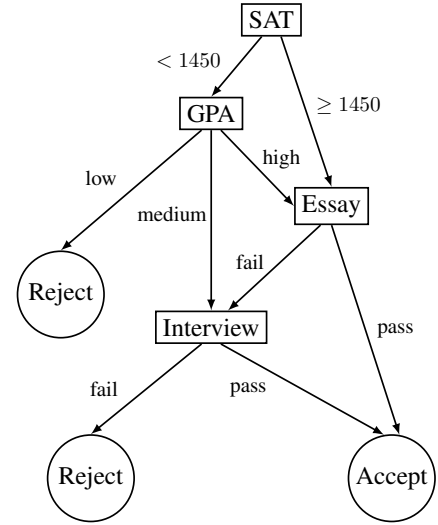

We next provide closed forms for the complete reasons of decision graphs, which subsume
decision trees, in the form of monotone, \(\vee\)-decomposable NNFs. This
will later facilitate the computation of their prime implicants 
and implicates (sufficient and necessary reasons).  
We first treat multi-class decision graphs with nominal features and then
treat decision graphs with numeric features; see Figures~\ref{fig:admissions-dg} and~\ref{fig:dt-dg}.

Each leaf node in a decision graph is labeled with some class~\(c\). Moreover,
each internal node \(T\) in the graph has outgoing edges \(\DE(X,S_1,T_1), \ldots, \DE(X,S_n,T_n)\), \(n \geq 2\).
We say in this case that node \(T\) {\em tests} variable \(X\). The children of node \(T\) are \(T_1, \ldots, T_n\)
and \(S_1, \ldots, S_n\) is a partition of {\em some} states of variable \(X\). 
A decision graph will be represented by its root node.
Hence, each node in the graph represents a smaller decision graph.
We allow variables to be tested more than once on a path from the root to a leaf but under
the following condition, which we call the {\em weak test-once property.} 
Consider path \(\ldots, T\DE(X,S_j,T_j), \ldots, T'\DE(X,R_k,T_k), \ldots\) from the root to leaf
and suppose that nodes \(T\) and \(T'\) test variable \(X\). If no nodes between \(T\)
and \(T'\) on the path test variable \(X\), then \(\{R_k\}_k\) must be a partition of states \(S_j\). Moreover,
if \(T\) is the first node that tests \(X\) on the path, then \(\{S_j\}_j\)
must be a partition of {\em all} states for \(X\).
For binary variables, the weak test-once property reduces to the standard test-once 
property: A variable can be tested at most once on any path from the root to a leaf. 
The weak test-once property is critical for treating numeric features.
As we show later, one can easily discretize continuous variables based on the thresholds used at decision nodes,
which leads to decision graphs that satisfy the weaker test-once property but not the standard one. 

A decision graph classifies an instance \(\delta\) as follows. Suppose \(\delta[X]\) is the
state of variable \(X\) in instance \(\delta\). We start at the graph root and repeat the following.
When we are at node \(T\) that has outgoing edges \(\DE(X,S_1,T_1), \ldots, \DE(X,S_n,T_n)\), we follow the
(unique) edge \(\DE(X,S_i,T_i)\) which satisfies \(\delta[X] \in S_i\). This process leads us to a unique leaf node. 
The label \(c\) of this leaf node is then the class assigned to instance \(\delta\) by the decision 
graph (that is, instance \(\delta\) belongs to class \(c\)).

We next provide a closed-form NNF that captures the instances belonging to some class \(c\) in a decision graph.

\begin{definition}\label{def:dt-nnf}
The NNF for a decision graph \(T\) and class \(c\) is denoted \(\Delta^c[T]\) and defined inductively as follows:
\begin{equation*}
\footnotesize
\Delta^c[T] =
\left\{
\begin{array}{ll}
\top & \mbox{if \(T\) has class \(c\)} \\
\bot & \mbox{if \(T\) has class \(c'  \neq c\)} \\
\bigwedge_j (\Delta^c[T_j] \vee \bigvee_{x_i \not \in S_j} x_i)
& \mbox{if \(T\) has edges \(\DE(X,S_j,T_j)\)}
\end{array}
\right.
\label{eq:dt-nnf} 
\end{equation*}
\end{definition}

\begin{proposition}\label{prop:dt-nnfs}
For decision graph \(T\), class \(c\) and instance \(\delta\), we have
\(\delta \models \Delta^c[T]\) iff \(T\) assigns class \(c\) to instance \(\delta\). 
\end{proposition}

This NNF is positive and can be constructed in linear time but 
is not \(\vee\)-decomposable: The disjuncts in \(\bigvee_{x_j \not \in S_i} x_j\) 
share variables and  variable \(X\) will appear in \(\Delta^c[T_j]\) if tested
again in graph \(T_j\). Yet, this NNF is tractable for universal quantification as revealed
in the proof of the next result, which provides closed-form complete reasons for decision graphs.
 
\begin{proposition}\label{prop:cr-cf}
Let \(T\) be a decision graph, \(\delta\) be an instance in class \(c\) and \(\delta[X]\) be the state of variable \(X\) in instance \(\delta\).
The complete reason \(\forall \delta \cdot \Delta^c[T]\) is given by the NNF:
\begin{equation}
\footnotesize
\Gamma^c[T] = 
\left\{
\begin{array}{ll}
\top & \mbox{if \(T\) has class \(c\)}  \\
\bot & \mbox{if \(T\) has class \(c' \neq c\)} \\
\bigwedge_{j} (\Gamma^c[T_j] \vee \l_j) & 
\mbox{if \(T\) has edges \(\DE(X,S_j,T_j)\)}
\end{array}
\right.
\label{eq:cr-cf} 
\end{equation}
where \(\l_j = \delta[X]\) if  \(\delta[X] \in S_k\) for some \(k \neq j\), else \(\l_j = \bot\).
\end{proposition}

 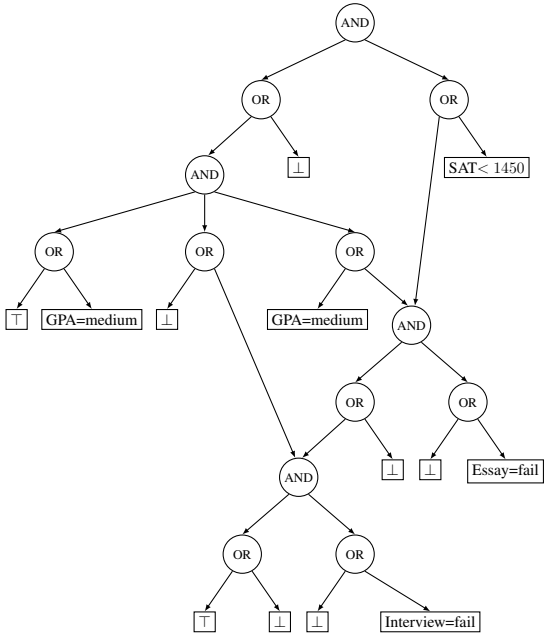
\begin{figure}[tb]
        \centering
        \scalebox{0.5}{
        \begin{tikzpicture}[
        roundnode/.style={circle ,draw=black, thick, inner sep=1pt, align=center, text width = 9mm},
        squarednode/.style={rectangle, draw=black, thick},
        ]
        \node[roundnode] (AND1) {AND};
        \node[roundnode] (OR1_1) [below = of AND1, xshift = -1.5cm] {OR};
        \node[roundnode] (OR1_2) [below = of AND1, xshift = 1.5cm] {OR};
        \draw[-latex, thick] (AND1.240) -- (OR1_1.90);
        \draw[-latex, thick] (AND1.300) -- (OR1_2.90);
        \node[squarednode] (T1_11) [below = of OR1_1, xshift = -1cm] {\Large $\top$};
        \node[squarednode] (T1_12) [below = of OR1_1, xshift = 1cm] {\Large $\bot$};
        \node[squarednode] (T1_21) [below = of OR1_2, xshift = -1cm] {\Large $\bot$};
        \node[squarednode] (T1_22) [below = of OR1_2, xshift = 2cm] {\Large Interview=fail};
        \draw[-latex, thick] (OR1_1.240) -- (T1_11.90);
        \draw[-latex, thick] (OR1_1.300) -- (T1_12.90);
        \draw[-latex, thick] (OR1_2.240) -- (T1_21.90);
        \draw[-latex, thick] (OR1_2.300) -- (T1_22.90);

        \node[roundnode] (AND2) [above = of AND1, xshift = 3cm, yshift = 2cm] {AND};
        \node[roundnode] (OR2_1) [below = of AND2, xshift = -1.5cm] {OR};
        \node[roundnode] (OR2_2) [below = of AND2, xshift = 1.5cm] {OR};
        \draw[-latex, thick] (AND2.240) -- (OR2_1.90);
        \draw[-latex, thick] (AND2.300) -- (OR2_2.90);
        \node[squarednode] (T2_12) [below = of OR2_1, xshift = 1cm] {\Large $\bot$};
        \node[squarednode] (T2_21) [below = of OR2_2, xshift = -1cm] {\Large $\bot$};
        \node[squarednode] (T2_22) [below = of OR2_2, xshift = 1cm] {\Large Essay=fail};
        \draw[-latex, thick] (OR2_1.240) -- (AND1.80);
        \draw[-latex, thick] (OR2_1.300) -- (T2_12.90);
        \draw[-latex, thick] (OR2_2.240) -- (T2_21.90);
        \draw[-latex, thick] (OR2_2.300) -- (T2_22.90);
        
        \node[roundnode] (AND3) [above = of AND1, xshift = -2.5cm, yshift = 6cm] {AND};
        \node[roundnode] (OR3_1) [below = of AND3, xshift = -4cm] {OR};
        \node[roundnode] (OR3_2) [below = of AND3] {OR};
        \node[roundnode] (OR3_3) [below = of AND3, xshift = 4cm] {OR};
        \draw[-latex, thick] (AND3.240) -- (OR3_1.90);
        \draw[-latex, thick] (AND3.270) -- (OR3_2.90);
        \draw[-latex, thick] (AND3.300) -- (OR3_3.90);
        \node[squarednode] (T3_11) [below = of OR3_1, xshift = -1cm] {\Large $\top$};
        \node[squarednode] (T3_12) [below = of OR3_1, xshift = 1cm] {\Large GPA=medium};
        \node[squarednode] (T3_21) [below = of OR3_2, xshift = -1cm] {\Large $\bot$};
        \node[squarednode] (T3_31) [below = of OR3_3, xshift = -1cm] {\Large GPA=medium};
        \draw[-latex, thick] (OR3_1.240) -- (T3_11.90);
        \draw[-latex, thick] (OR3_1.300) -- (T3_12.90);
        \draw[-latex, thick] (OR3_2.240) -- (T3_21.90);
        \draw[-latex, thick] (OR3_2.300) -- (AND1.100);
        \draw[-latex, thick] (OR3_3.240) -- (T3_31.90);
        \draw[-latex, thick] (OR3_3.300) -- (AND2.100);
        
        \node[roundnode] (AND4) [above = of AND3, xshift = 4cm, yshift = 2cm] {AND};
        \node[roundnode] (OR4_1) [below = of AND4, xshift = -2.5cm] {OR};
        \node[roundnode] (OR4_2) [below = of AND4, xshift = 2.5cm] {OR};
        \draw[-latex, thick] (AND4.240) -- (OR4_1.90);
        \draw[-latex, thick] (AND4.300) -- (OR4_2.90);
        \node[squarednode] (T4_12) [below = of OR4_1, xshift = 1cm] {\Large $\bot$};
        \node[squarednode] (T4_22) [below = of OR4_2, xshift = 1cm] {\Large SAT$<1450$};
        \draw[-latex, thick] (OR4_1.240) -- (AND3.80);
        \draw[-latex, thick] (OR4_1.300) -- (T4_12.90);
        \draw[-latex, thick] (OR4_2.240) -- (AND2.80);
        \draw[-latex, thick] (OR4_2.300) -- (T4_22.90);
        \end{tikzpicture}
        }
      \caption{A complete reason constructed by Proposition~\ref{prop:cr-cf}
      for the decision graph in Figure~\ref{fig:admissions-dg} and instance
      SAT~$<1450$, GPA=medium, Essay=fail, Interview=fail. 
      This complete reason is in the form of a monotone, \(\vee\)-decomposable NNF. 
      \label{fig:admissions-cr}}
    \end{figure}

Consider the decision graph in Figure~\ref{fig:admissions-dg} and an applicant who
scored \(< 1450\) on the SAT, had a medium GPA and did not pass their essay or interview.
This applicant is rejected by the classifier and the complete reason for the decision, 
as constructed by Proposition~\ref{prop:cr-cf}, is shown in Figure~\ref{fig:admissions-cr}.

\begin{proposition}\label{prop:dt-cr}
Let \(T\) be a decision graph and \(\delta\) be an instance in class \(c\).
The complete reason \(\forall \delta \cdot \Delta^c[T]\) in Equation~\ref{eq:cr-cf} is 
an NNF that is monotone and \(\vee\)-decomposable.
\end{proposition}

Even though we are working
with decision graphs that include discrete variables and multiple classes, we get complete
reasons in the form of monotone NNFs, which are effectively Boolean NNFs.
This will simplify the computation of necessary and sufficient reasons in later sections,
as it allows us to avoid certain complications that can arise when binarizing discrete variables;
see, e.g.,~\cite{corr/abs-2007-01493}.

\subsection{Numeric Features}

\begin{figure}[tb]
 \centering
 \includegraphics[height=0.34\textwidth]{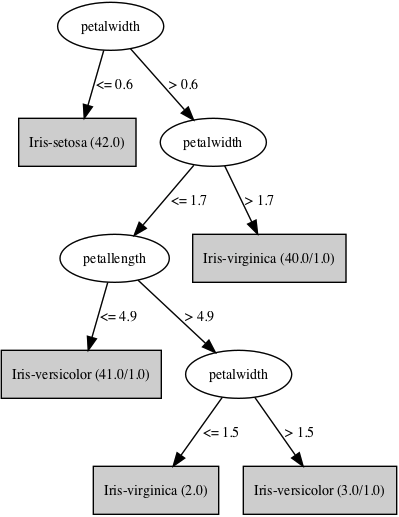} \quad
  \includegraphics[height=0.34\textwidth]{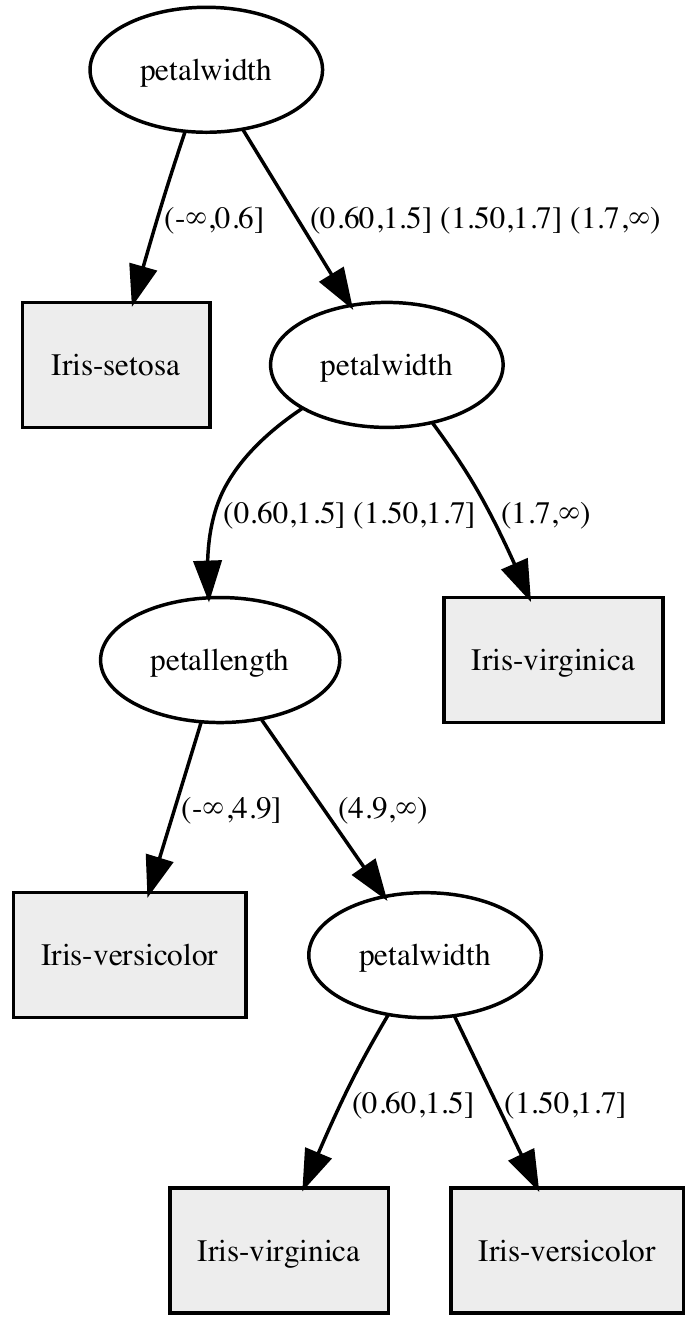}
\caption{A decision tree with continuous variables learned using weka (left) 
and its discretization (right). \label{fig:dt-dg}}
\end{figure}

Suppose we have a continuous variable \(X\) that is being tested at node \(T\) in a decision graph. The test will have
the form \(X \leq t_i\), where \(t_i\) is a threshold in \((-\infty,\infty)\). Node \(T\) will then have two outgoing edges,
one is followed when \(X \leq t_i\) (high edge) and the other is followed when \(X > t_i\) (low edge);
see Figure~\ref{fig:dt-dg}. 
Suppose now that \(t_1, \ldots, t_n\) is the set of all thresholds for variable \(X\) in the decision graph 
and assume that these thresholds are in increasing order. We can then treat variable \(X\)
as a discrete variable with the following \(n+1\) states: \((-\infty,t_1], (t_1,t_2], \ldots, (t_{n-1},t_n], (t_n,\infty)\). 
If variable \(X\) is being tested first at node \(T\), 
we label the high edge of node \(T\) with states \((-\infty,t_1], (t_1,t_2], \ldots, (t_{i-1},t_i]\)
and its low edge with states \((t_i,t_{i+1}], \ldots, (t_{n-1},t_n], (t_n,\infty)\). 
Consider Figure~\ref{fig:dt-dg} (left). 
Variable ``petalwidth" (\(W\)) has three thresholds \(0.6, 1.5, 1.7\), leading to four discrete states 
\(S_W = (-\infty, 0.6]\), \((0.6, 1.5]\), \((1.5, 1.7]\), \((1.7,\infty)\).
Variable ``petallength'' (\(L\)) has one threshold \(4.9\), leading to
two discrete states \(S_L = (-\infty, 4.9]\), \((4.9, \infty)\). 
Variable \(W\) is tested three times in the decision tree.
The first test (\(W \leq 0.6\)) splits states \(S_W\) into
\(S_1 = (-\infty, 0.6]\) for the high edge and \(S_2 = (0.6, 1.5], (1.5, 1.7], (1.7,\infty)\) for the low edge. 
The second test (\(W \leq 1.7\)) splits \(S_2\) into 
\(S_{21} = (0.6, 1.5], (1.5, 1.7]\) for the high edge and \(S_{22} = (1.7,\infty)\) for the low edge.
The third and final test (\(W \leq 1.5\)) splits states \(S_{21}\) into \((0.6, 1.5]\) and \((1.5, 1.7]\). 
The resulting decision tree with discrete variables
{\em does not} satisfy the test-once property but does satisfy the weak test-once property as shown in
Figure~\ref{fig:dt-dg}(right). 

Consider now instance \(\delta_1: W = 0.8, L = 5.3\) which is classified as ``Iris-virginica'' by the decision tree
with continuous variables (\(T_1\)). We can view this instance as the discrete instance \(\delta_2: W = (0.6,1.5], L = (4.9,\infty)\)
since \(0.8 \in (0.6,1.5]\) and \(5.3 \in (4.9,\infty)\).
The decision tree with discrete variables (\(T_2\)) will also classify instance \(\delta_2\) as ``Iris-virginica.''
A continuous instance and its corresponding 
discrete instance will be classified identically by decision trees \(T_1\) and \(T_2\) because
 \(T_1\) cannot discriminate continuous values that
belong to the same  interval. Finally, to generate the complete reason for instance \(\delta_1\), we compute 
\(\forall \delta_2 \cdot \Delta^c[T_2]\) using Proposition~\ref{prop:cr-cf} where \(c\) is class ``Iris-virginica.''

\subsection{Further Extensions}

The closed-form complete reason in Proposition~\ref{prop:cr-cf} applies directly to {\em Free Binary Decision Diagrams (FBDDs)}~\cite{tc/GergovM94} 
and {\em Ordered Binary Decision Diagrams (OBDDs)}~\cite{tc/Bryant86} 
as they are special cases of decision graphs. FBDDs use binary variables and binary classes (\(\top\) and \(\bot\)).
OBDDs are a subset of FBDDs which test variables in the same order along any
path from the root to a leaf. 
We can similarly obtain closed forms for the complete reasons of {\em Sentential Decision Diagrams (SDDs)}~\cite{ijcai/Darwiche11}, 
which test on formulas (sentences) instead of variables. 
This is possible since given an SDD for \(\Delta\) we can obtain an SDD for \(\neg \Delta\) in linear time.
An SDD \(\Delta\) is an \(\wedge\)-decomposable NNF that represents instances for  class \(\top\).
The SDD for \(\neg \Delta\) is also an \(\wedge\)-decomposable NNF but represents
instances for class \(\bot\). If we negate \(\Delta\) and \(\neg \Delta\) using deMorgan's law, we obtain 
\(\vee\)-decomposable NNFs for classes \(\bot\) and \(\top\), respectively. 
This allows us to obtain a closed-form, monotone, \(\vee\)-decomposable complete reason for any instance
using universal quantification.
\cite{darwiche2021quantifying} showed that an SDD can be universally quantified in linear time. 
Earlier, \cite{ecai/DarwicheH20} showed that Decision-DNNFs~\cite{jair/HuangD07}
can be universally quantified in linear time as well.\footnote{\cite{ecai/DarwicheH20} introduced two linear-time operations on Decision-DNNFs:
{\em consensus} and {\em filtering.} These operations implement
universal literal quantification as shown in~\cite{darwiche2021quantifying}. Decision-DNNFs are \(\wedge\)-decomposable
NNFs in which disjunctions have the form \((x \wedge \alpha) \vee (\n{x} \wedge \beta)\).}
Decision-DNNFs cannot be negated efficiently so they do not permit 
closed-form complete reasons unless we have Decision-DNNFs for classes \(\top\) and \(\bot\).
While decision tree classifiers are normally learned from data, classifiers such as OBDDs and SDDs 
are compiled from other classifiers like Bayesian/neural networks and random forests; 
see, e.g.,~\cite{aaai/ShihCD19,kr/ShiSDC20,corr/abs-2007-01493}.
The relative succinctness of these representations of classifiers has been well studied.
FBDDs are a subset of Decision-DNNFs and there is a quasipolynomial simulation of Decision-DNNFs by equivalent FBDDs~\cite{uai/BeameLRS13}.
SDDs and FBDDs are not comparable~\cite{uai/BeameL15,mst/BolligB19} so SDDs and Decision-DNNFs are not comparable either.
SDDs are exponentially more succinct than OBDDs~\cite{aaai/Bova16}.

\section{Necessary and Sufficient Reasons}
\label{sec:nr-sr}

As mentioned earlier, the prime implicants of a complete reason can be interpreted as {\em sufficient reasons} for the decision.
\shrink{and correspond to PI explanations~\cite{ijcai/ShihCD18} and abductive explanations~\cite{IgnatievNM19a}.} 
We next show that the prime implicates of a complete reason can be interpreted as {\em necessary reasons} for the decision
and correspond to contrastive explanations~\cite{aiia/IgnatievNA020}.
We first provide further insights into complete reasons which will help in justifying this interpretation.

\begin{definition}\label{def:congruent}
Instances \(\delta_1\) and \(\delta_2\) are \hl{congruent} iff \(\delta_1 \cap \delta_2 \models \Delta\) for some class formula \(\Delta\).
We also say in this case that the decisions on instances \(\delta_1\) and \(\delta_2\) are congruent. 
\end{definition}
If instances \(\delta_1\) and \(\delta_2\) are congruent, they must belong
to the same class since \(\delta_1 \models \Delta\) and \(\delta_2 \models \Delta\) so they are decided similarly.
Moreover, their common characteristics \(\delta_1\cap\delta_2\) are sufficient to justify the decision.
That is, the decisions on them are equal and have a common justification.

\begin{proposition}\label{prop:congruent} 
Let \(\forall \delta \cdot \Delta\) be the complete reason for the decision on instance \(\delta\).
Then instance \(\delta^\star\) is congruent to instance \(\delta\) iff \(\delta^\star \models \forall \delta \cdot \Delta\). 
\end{proposition}

\shrink{
\begin{corollary}\label{coro:congruent}
Two decisions are congruent iff their complete reasons are consistent with one another.
\end{corollary}
\begin{proof}[\bf Proof of Proposition~\ref{}]
\end{proof}
}

Hence, the complete reason \(\forall \delta \cdot \Delta\) captures all, and only, instances 
that are congruent to instance \(\delta\). The complete reason does not capture all
instances that are decided similarly to \(\delta\) since some of these instances may be decided 
that way for a different reason (the decisions are not congruent).

Consider the class formula \(\Delta = \n{x}_1(x_2+\n{y}_1)(\n{y}_1+z_1)\) over ternary variables \(X\), \(Y\) and \(Z\).
The instance \(\delta = x_2 y_2 z_1\) is decided positively by this formula (\(\delta \models \Delta\)) 
and the complete reason for this decision is \(\forall x_2, y_2, z_1 \cdot \Delta = x_2(y_2+z_1)\). 
There are four other instances that satisfy this complete reason, 
\(x_2 y_2 z_2\), \(x_2 y_2 z_3\), \(x_2 y_1 z_1\) and \(x_2 y_3 z_1\). 
All are decided positively by \(\Delta\) and the states each share with instance \(\delta\) justify the decision.
Instance \(x_3 y_2 z_1\) is also decided positively by \(\Delta\) but for a different reason:
the states \(y_2 z_1\) it shares with instance \(\delta\) do not justify the decision, \(y_2 z_1 \not \models \Delta\). 
Hence, this instance is not captured by the complete reason for \(\delta\).

\subsection{Implicants and Implicates as Reasons}

We next review the interpretation of prime implicants as sufficient reasons and discuss the 
interpretation of prime implicates as necessary reasons for a decision. We will represent these notions by sets
of literals, which are interpreted as conjunctions for prime implicants (terms) and
as disjunctions for prime implicates (clauses). 

\begin{proposition}\label{prop:subsets}
The prime implicants and prime implicates of a complete reason \(\forall \delta \cdot \Delta\) are 
subsets of instance \(\delta\). 
\end{proposition}

A prime implicant \(\sigma\) of the complete reason \(\forall \delta \cdot \Delta\) can be viewed as a sufficient reason for 
the underlying decision as it is a minimal subset of instance \(\delta\) that is guaranteed to sustain the decision, congruently.
If we change any part of the instance but for \(\sigma\), the decision will stick and for a common reason
since the new and old instances are congruent.
Consider the complete reason in Figure~\ref{fig:admissions-cr} which corresponds to a reject decision on the
instance SAT~$<1450$, GPA=medium, Essay=fail, Interview=fail. 
There are two prime implicants for this complete reason: 
\(\{\text{SAT~$<1450$, GPA=medium, Interview=fail}\}\) and 
\(\{\text{Essay=fail, Interview=fail}\}\). 
Each of these prime implicants is a minimal subset of the instance that is sufficient to trigger the reject decision.

A prime implicate \(\sigma\) of the complete reason \(\forall \delta \cdot \Delta\) can be viewed as a necessary reason for 
the underlying decision as it is a minimal subset of the instance that is essential for sustaining a congruent decision. 
If we change all states in \(\sigma\), the decision on the new instance will be different or will be made for a different 
reason since the new and old instances will not be congruent (we provide a stronger semantics later). 
Consider again the complete reason in Figure~\ref{fig:admissions-cr} and the corresponding instance and reject decision.
There are three prime implicates for this complete reason:
\(\{\text{Interview=fail}\}\), 
\(\{\text{SAT~$<1450$, Essay=fail}\}\) and 
\(\{\text{GPA=medium, Essay=fail}\}\).
Changing Interview to pass will change the decision. Changing SAT to~$\geq 1450$ and Essay to pass will also change
the decision. Since GPA is a ternary variable, there are two ways to change its value.
If we change GPA and Essay to high and pass, respectively, the decision will change. But if we change these features
to low and pass, respectively, the decision will not change but the new instance (SAT~$<1450$, GPA=low, Essay=pass, Interview=fail) 
will not be congruent with the original instance (SAT~$<1450$, GPA=medium, Essay=fail, Interview=fail). 
That is, the common characteristics of these instances 
\(\{\text{SAT~$<1450$, Interview=fail}\}\) cannot on their own justify the reject decision.

For yet another example, consider again class formula \(\Delta = \n{x}_1(x_2+\n{y}_1)(\n{y}_1+z_1)\) over ternary variables \(X\), \(Y\) and \(Z\).
The complete reason for positive instance \(\delta = x_2 y_2 z_1\) is
\(\Gamma = \forall \delta \cdot \Delta = x_2(y_2+z_1)\). 
The prime implicants of \(\Gamma\) are \(x_2 y_2\) and \(x_2 z_1\), which are the sufficient reasons for the decision.
If we change instance \(\delta\) while keeping one of these reasons intact, the decision sticks.
The prime implicates of \(\Gamma\) are \(x_2\) and \(y_2+z_1\), which are the necessary reasons for the decision.
If we violate one of these reasons, the decision will be different or made for a different reason.
Changing instance \(\delta\) to \(x_2 y_1 z_3\) violates the necessary reason \(y_2+z_1\), which leads to a negative decision. 
Changing the instance to  \(\delta^\star = x_3 y_2 z_1\) violates the necessary reason \(x_2\). 
The decision remains positive though but for a different reason than why \(\delta\) is positive. 
That is, the common characteristics \(\delta \cap \delta^\star = y_2 z_1\) do not justify the decision
on these instances, \(y_2 z_1 \not \models \Delta\).

\subsection{More on Necessity}

We next show that necessary reasons correspond to {\em basic} contrastive explanations as
formalized in~\cite{aiia/IgnatievNA020} using the following definition (modulo notation).

\begin{definition}\label{def:contrastive}
Let \(\delta\) be an instance decided positively by class formula \(\Delta\).
A \hl{contrastive explanation} of this decision is a minimal subset \(\gamma\) of instance \(\delta\)
such that \(\delta \setminus \gamma \not \models \Delta\).
\end{definition}
That is, it is possible to change the decision on instance \(\delta\) by {\em only} changing the states in \(\gamma\).
Moreover, we must change {\em all} states in \(\gamma\) for the decision to change. 

\begin{proposition} \label{prop:contrastive}
Let \(\delta\) be an instance decided positively by class formula \(\Delta\).
Then \(\gamma\) is a prime implicate of the complete reason \(\forall \delta \cdot \Delta\)
iff \(\gamma\) is a contrastive explanation.
\end{proposition}

This correspondence is perhaps not too
surprising given the duality between abductive and contrastive explanations~\cite{aiia/IgnatievNA020}
and the classical duality between prime implicants and prime implicates. 
However, it does provide further insights into contrastive explanations: changing the
states of a contrastive explanation leads to a non-congruent decision.
It also provides further insights on the necessity of prime implicates:
while violating a necessary reason will only lead to an instance
that is not congruent (decided differently or for a different reason), there must exist at least one violation 
of each necessary reason which is guaranteed to change the decision. 
This follows directly from Definition~\ref{def:contrastive}. If the variables of a necessary
reason are all binary, there is only one way to violate the reason (by negating each
variable in the reason). In this case, violating the necessary reason is
guaranteed to change the decision. 

For an example, let us revisit 
the complete reason in Figure~\ref{fig:admissions-cr} and the corresponding instance and reject decision.
This decision has three necessary reasons: 
\(\{\text{Interview=fail}\}\), 
\(\{\text{SAT~$<1450$, Essay=fail}\}\) and 
\(\{\text{GPA=medium, Essay=fail}\}\).
There is only one way to violate each of the first two reasons, and each violation leads to reversing the
decision as we saw earlier. There are two ways to violate the third necessary reason. One of these violations
(GPA=high, Essay=pass) reverses the decision but the other violation (GPA=low, Essay=pass) 
keeps the reject decision intact (but for a different reason).

In summary, a necessary reason (contrastive explanation) identifies a minimal subset of the instance 
which is guaranteed to change the decision if that subset is altered {\em properly.} 
The minimality condition ensures that we must alter every variable in a necessary reason to change the decision,
but it does not specify how to alter it (except for binary variables). We will revisit this distinction
when we discuss counterfactual explanations~\cite{kr/AudemardKM20}.

\subsection{Targeting a Particular Class}

\begin{figure}[tb]
	\centering
	
	\scalebox{0.6}{
		\begin{tikzpicture}[
			roundnode/.style={circle ,draw=black, thick},
			squarednode/.style={rectangle, draw=black, thick},
			]
			\node[squarednode] (X) {\Large $X$};
			\node[squarednode] (Y) [below = of X, xshift = 1cm] {\Large $Y$};
			\node[roundnode] (C1) [below = of X, xshift = -1cm] {\Large $c_1$};
			\node[roundnode] (C2) [below = of Y, xshift = -1cm] {\Large $c_2$};
			\node[roundnode] (C3) [below = of Y, xshift = 1cm] {\Large $c_3$};
			
			\draw[-latex, thick] (X.240) -- node [anchor = center, xshift = -2mm, yshift = 2mm] {\Large$x_1$} (C1.90);
			\draw[-latex, thick] (X.300) -- node [anchor = center, xshift = 2mm, yshift = 2mm] {\Large$x_2$} (Y.90);
			\draw[-latex, thick] (Y.240) -- node [anchor = center, xshift = -2mm, yshift = 2mm] {\Large$y_1$} (C2.90);
			\draw[-latex, thick] (Y.300) -- node [anchor = center, xshift = 2mm, yshift = 2mm] {\Large $y_2$} (C3.90);
		\end{tikzpicture}
	}
	\quad
	\scalebox{0.37}{
		\begin{tikzpicture}[
			roundnode/.style={circle ,draw=black, thick, inner sep=1pt, align=center, text width = 9mm},
			squarednode/.style={rectangle, draw=black, thick},
			]
			\node[roundnode] (AND1) {AND};
			\node[roundnode] (OR1_1) [below = of AND1, xshift = -1.5cm] {OR};
			\node[roundnode] (OR1_2) [below = of AND1, xshift = 1.5cm] {OR};
			\draw[-latex, thick] (AND1.240) -- (OR1_1.90);
			\draw[-latex, thick] (AND1.300) -- (OR1_2.90);
			\node[squarednode] (T1_11) [below = of OR1_1, xshift = -1cm] {\huge$\top$};
			\node[squarednode] (T1_12) [below = of OR1_1, xshift = 1cm] {\huge$\bot$};
			\node[squarednode] (T1_21) [below = of OR1_2, xshift = -1cm] {\huge$\bot$};
			\node[squarednode] (T1_22) [below = of OR1_2, xshift = 1cm] {\Huge$y_1$};
			\draw[-latex, thick] (OR1_1.240) -- (T1_11.90);
			\draw[-latex, thick] (OR1_1.300) -- (T1_12.90);
			\draw[-latex, thick] (OR1_2.240) -- (T1_21.90);
			\draw[-latex, thick] (OR1_2.300) -- (T1_22.90);

			\node[roundnode] (AND2) [above = of AND1, xshift = -1cm, yshift = 2cm] {AND};
			\node[roundnode] (OR2_1) [below = of AND2, xshift = -1.5cm] {OR};
			\node[roundnode] (OR2_2) [below = of AND2, xshift = 1.5cm] {OR};
			\draw[-latex, thick] (AND2.240) -- (OR2_1.90);
			\draw[-latex, thick] (AND2.300) -- (OR2_2.90);
			\node[rectangle, draw=black, ultra thick] (T2_11) [below = of OR2_1, xshift = -1cm] {\huge$\bot$};
			\node[squarednode] (T2_12) [below = of OR2_1, xshift = 1cm] {\Huge$x_2$};
			\node[squarednode] (T2_22) [below = of OR2_2, xshift = 1cm] {\huge $\bot$};
			\draw[-latex, thick] (OR2_1.240) -- (T2_11.90);
			\draw[-latex, thick] (OR2_1.300) -- (T2_12.90);
			\draw[-latex, thick] (OR2_2.240) -- (AND1.90);
			\draw[-latex, thick] (OR2_2.300) -- (T2_22.90);
		\end{tikzpicture}
	}
	\quad
	\scalebox{0.37}{
		\begin{tikzpicture}[
			roundnode/.style={circle ,draw=black, thick, inner sep=1pt, align=center, text width = 9mm},
			squarednode/.style={rectangle, draw=black, thick},
			]
			\node[roundnode] (AND1) {AND};
			\node[roundnode] (OR1_1) [below = of AND1, xshift = -1.5cm] {OR};
			\node[roundnode] (OR1_2) [below = of AND1, xshift = 1.5cm] {OR};
			\draw[-latex, thick] (AND1.240) -- (OR1_1.90);
			\draw[-latex, thick] (AND1.300) -- (OR1_2.90);
			\node[squarednode] (T1_11) [below = of OR1_1, xshift = -1cm] {\huge $\top$};
			\node[squarednode] (T1_12) [below = of OR1_1, xshift = 1cm] {\huge$\bot$};
			\node[squarednode] (T1_21) [below = of OR1_2, xshift = -1cm] {\huge$\bot$};
			\node[squarednode] (T1_22) [below = of OR1_2, xshift = 1cm] {\Huge$y_1$};
			\draw[-latex, thick] (OR1_1.240) -- (T1_11.90);
			\draw[-latex, thick] (OR1_1.300) -- (T1_12.90);
			\draw[-latex, thick] (OR1_2.240) -- (T1_21.90);
			\draw[-latex, thick] (OR1_2.300) -- (T1_22.90);

			\node[roundnode] (AND2) [above = of AND1, xshift = -1cm, yshift = 2cm] {AND};
			\node[roundnode] (OR2_1) [below = of AND2, xshift = -1.5cm] {OR};
			\node[roundnode] (OR2_2) [below = of AND2, xshift = 1.5cm] {OR};
			\draw[-latex, thick] (AND2.240) -- (OR2_1.90);
			\draw[-latex, thick] (AND2.300) -- (OR2_2.90);
			\node[rectangle, draw=black, ultra thick] (T2_11) [below = of OR2_1, xshift = -1cm] {\huge$\top$};
			\node[squarednode] (T2_12) [below = of OR2_1, xshift = 1cm] {\Huge$x_2$};
			\node[squarednode] (T2_22) [below = of OR2_2, xshift = 1cm] {\huge$\bot$};
			\draw[-latex, thick] (OR2_1.240) -- (T2_11.90);
			\draw[-latex, thick] (OR2_1.300) -- (T2_12.90);
			\draw[-latex, thick] (OR2_2.240) -- (AND1.90);
			\draw[-latex, thick] (OR2_2.300) -- (T2_22.90);
		\end{tikzpicture}
	}
	
	\caption{(a) decision tree (b) complete reason for ``why \(x_2 y_1\) is \(c_2\)'' (c) complete reason for ``why \(x_2 y_1\) is not \(c_3\).''
		\label{fig:target}
	}
\end{figure}
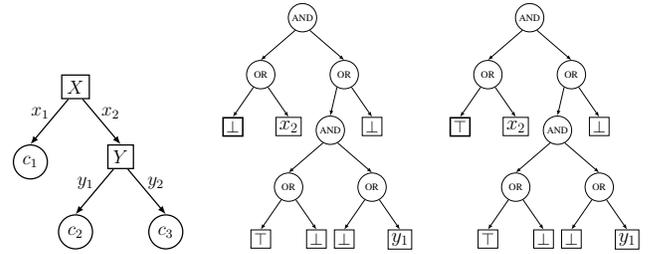

Beyond basic contrastive explanations, 
\cite{aiia/IgnatievNA020} discussed {\em targeted} contrastive explanations which aim to change
the instance class from \(c\) to some class \(c^\star\); see~\cite{Lipton1990-LIPCE,ai/Miller19}.
This notion is particularly relevant to multi-class classifiers as it reduces to basic 
contrastive explanations when the classifier has only two classes. 
Targeted contrastive explanations
can be obtained using the complete reason for {\em why} the instance was classified as 
{\em not}~\(c^\star\) (that is, a class other than \(c^\star\)). 
This complete reason can be obtained using a slight modification of Equation~\ref{eq:cr-cf} where we modify the first two conditions as follows:
\begin{equation}
\footnotesize
\Gamma^c[T] = 
\left\{
\begin{array}{ll}
\top & \mbox{if \(T\) has class \(c' \neq c^\star\)}  \\
\bot & \mbox{if \(T\) has class \(c^\star\)} \\
\bigwedge_{j} (\Gamma^c[T_j] \vee \l_j) & 
\mbox{if \(T\) has edges \(\DE(X,S_j,T_j)\)}
\end{array}
\right.
\label{eq:cr-cf-target} 
\end{equation}
The  prime implicates for this complete reason (i.e., necessary reasons) will then identify minimal subsets of the instance that 
lead to the targeted class \(c^\star,\) if altered properly. 

Consider the decision tree in Figure~\ref{fig:target}(a) which has two binary features \(X\) and \(Y\) and 
three classes \(c_1, c_2, c_3.\) The instance \(x_2 y_1\) is classified as \(c_2\). The complete reason for this
decision, as computed by Equation~\ref{eq:cr-cf}, is shown in Figure~\ref{fig:target}(b) and
has two necessary reasons \(x_2\) and \(y_1\). If we violate the first necessary reason (\(x_2 \rightarrow x_1\)), the class changes to \(c_1\).
If we violate the second necessary reason (\(y_1 \rightarrow y_2\)), the class changes to \(c_3\). 
Suppose now we wish to change the class \(c_2\) of this instance particularly to \(c_3\). The complete reason for ``why not \(c_3\),'' 
as computed by Equation~\ref{eq:cr-cf-target}, is shown in Figure~\ref{fig:target}(c) and
has only one necessary reason, \(y_1\). Violating this necessary reason is guaranteed to change the class to \(c_3\).

\cite{kr/AudemardKM20} discussed the complexity of computing the related notion of {\em counterfactual explanations} which are defined as follows.
Given an instance~\(\delta\) in class \(c\), find an instance \(\delta^\star\) in a different class \(c^\star\) that is as close as possible 
to instance \(\delta\) with respect to the hamming distance. In other words, instance \(\delta^\star\) must maximize the number of characteristics it shares with instance~\(\delta\). 
Consider now the characteristics \(\gamma\) of instance \(\delta\) that do not appear in instance \(\delta^\star\) (\(\gamma=\delta\setminus\delta^\star\)). 
Changing these characteristics to \(\gamma^\star = \delta^\star\setminus\delta\) will change the class from \(c\) to \(c^\star\).
Hence, characteristics \(\gamma\) are a {\em length-minimal} subset of instance \(\delta\) which, if changed properly, will guarantee a change from class \(c\) to class \(c^\star\).
Every characteristic of \(\gamma\) must be changed to ensure this class change, otherwise \(\delta^\star\) would not be a counterfactual explanation.
Moreover, when the features are binary, there is only one way to change the characteristics \(\gamma\) so the class will change from \(c\) to \(c^\star\);
that is, by flipping every characteristic in \(\gamma\) to yield \(\gamma^\star\).
In this case, counterfactual explanations are in one-to-one correspondence with the shortest necessary reasons which we discuss next.

\section{Computing Shortest Explanations}
\label{sec:algorithms}

A complete reason may have too many prime implicants and implicates. We will therefore
provide algorithms for computing the {\em shortest} implicants and 
implicates (which must be prime) for monotone, \(\vee\)-decomposable NNFs.
As discussed earlier, we can in linear time obtain complete reasons in this form
for decision graphs and SDDs, which include FBDDs, OBDDs and decision
trees as special cases. 

\begin{algorithm}[tb]
\footnotesize
\caption{Shortest Necessary Reasons (SNRs)
\label{alg:snr}}
\begin{algorithmic}[1]
\Require monotone and \(\vee\)-decomposable NNF \(\Delta\) with no constants
\Ensure all shortest implicates of NNF \(\Delta\)
\Function{snr}{$\Delta$} \Comment{\Call{Cache}{} initialized to \Call{nil}{}}
\If{\Call{cache}{$\Delta$} $\neq$ \Call{nil}{}}
	\Return \Call{cache}{$\Delta$}
\shrink{
\ElsIf{$\Delta = \bot$}
    \(snr \gets \{\{\}\}\)
\ElsIf{$\Delta = \top$}
    \(snr \gets \{\}\)
}
\ElsIf{$\Delta = x_i$} \(snr \gets \{\{x_i\}\}\)
\ElsIf{$\Delta = \alpha_1 \vee \ldots \vee \alpha_n$}
    \State \(snr \gets \Call{snr}{\alpha_1} \times \ldots \times \Call{snr}{\alpha_n}\) \label{snr:ln:or}
\Else{$\:\: \Delta = \alpha_1 \wedge \ldots \wedge \alpha_n$}
    \State \(snr \gets \bigcup_{ \Call{iml}{\alpha_i}= \Call{iml}{\Delta}} \Call{snr}{\alpha_i}\) \label{snr:ln:and}
\EndIf
\State $\Call{cache}{\Delta} \gets snr$
\State \Return $snr$ 
\EndFunction
\end{algorithmic}
\end{algorithm}

\begin{algorithm}[tb]
\footnotesize
\caption{Shortest Sufficient Reasons (SSRs)
\label{alg:ssr}}
\begin{algorithmic}[1]
\Require monotone and \(\vee\)-decomposable NNF \(\Delta\) with no constants
\Ensure all shortest implicants of NNF \(\Delta\)
\Function{ssr}{$\Delta$} \Comment{\Call{Cache}{} initialized to \Call{nil}{}}
\State \(k \gets 0\)
\Repeat
	\State \(ssr = \Call{imp}{\Delta,k,\{\}}\)
	\State \(k \gets k + 1\)
\Until{\(ssr \neq \{\}\)}
\State \Return \(ssr\)
\EndFunction
\Function{imp}{$\Delta$, $k$, $\sigma$} \Comment{computes \(k\)-implicants for \(\Delta | \sigma\)}
\If{ \(\Call{cache}{\Delta,k,\sigma} \neq \Call{nil}{}\)}  \label{ssr:ln:l-cache}
	\Return \(\Call{cache}{\Delta,k,\sigma}\)
\EndIf
\shrink{
\If{$\Delta = \bot$}
    \Return \{\}
\EndIf
\If{$\Delta = \top$}
    \Return \{\{\}\}
\EndIf
}
\State \(\Sigma \gets \{\}\)
\If{$\Delta = x$} 
    \If{\(x \in \sigma\)} \(\Sigma \gets \{\{\}\}\) \Comment{\(\Delta | \sigma\) valid}
    \ElsIf{\(k \geq 1\)} \(\Sigma \gets \{\{x\}\}\) 
    \EndIf
\ElsIf{$\Delta = \alpha_1 \vee \ldots \vee \alpha_n$}
	\For{\(\Sigma_i \gets \Call{imp}{\alpha_i,k,\sigma}\)}
		\If{\(\Sigma_i \neq \{\{\}\}\)}
			\(\Sigma \gets \Sigma \cup \Sigma_i\) \Comment{\(\alpha_i | \sigma\) not valid}
		\Else
			\(\:\: \Sigma \gets \{\{\}\}\); \textbf{break} \Comment{\(\Delta | \sigma\) valid}
		\EndIf
	\EndFor
\Else{$\:\:\Delta = \alpha_1 \wedge \ldots \wedge \alpha_n$}
    \For{\(\sigma_1 \in \Call{imp}{\alpha_1,k,\sigma}\)}
    	\For{\(\sigma_2 \in \Call{imp}{\alpha_2 \wedge \ldots \wedge \alpha_n, k-|\sigma_1|, \sigma \cup \sigma_1}\)}
    		\State \(\Sigma \gets \Call{remove\_subsumed}{\Sigma \cup \{\sigma_1 \cup \sigma_{2}\}}\)
    	\EndFor
    \EndFor
\EndIf
\State \(\Call{cache}{\Delta,k,\sigma} \gets \Sigma\) \label{ssr:ln:i-cache}
\State \Return \(\Sigma\)
\EndFunction
\end{algorithmic}
\end{algorithm}


{\bf Shortest Necessary Reasons.}
This will be an output polynomial algorithm that is based on three (conceptual) passes on the 
complete reason which we describe next.

\begin{definition}\label{def:iml}
The \hl{implicate minimum length (\iml)} of a valid  formula is \(\infty\).
For non-valid formulas, it is the minimum length attained by any implicate of the formula.
\end{definition}

The first pass computes the implicate minimum length.

\begin{proposition}\label{prop:iml}
The \iml\ of a monotone, \(\vee\)-decomposable NNF is computed as follows:
\(\iml(\top) = \infty\),
\(\iml(\bot) = 0\),
\(\iml(x_i) = 1\), 
\(\iml(\alpha \vee \beta) = \iml(\alpha) + \iml(\beta)\) and
\(\iml(\alpha \wedge \beta) = \min(\iml(\alpha),\iml(\beta))\).
\end{proposition}

The second pass prunes the NNF using the \iml\ of nodes.

\begin{proposition}\label{prop:prune}
Let \(\EXP(\Delta)\) be the NNF obtained from monotone, \(\vee\)-decomposable NNF \(\Delta\) 
by dropping \(\alpha_i\) from conjunctions \(\alpha = \alpha_1 \wedge \ldots \wedge \alpha_n\)
if \(\iml(\alpha_i) > \iml(\alpha)\).
Then \(\EXP(\Delta)\) is a monotone, \(\vee\)-decomposable NNF and
its prime implicates are the shortest implicates of \(\Delta\).
\end{proposition}

The third pass computes the prime implicates of NNF \(\EXP(\Delta)\) 
in output polynomial time.
Algorithm~\ref{alg:snr} implements the second and third passes assuming the first,
linear-time pass has been performed. It represents an implicate by a set of literals 
and uses the Cartesian product operation on sets of implicates: 
\(S_1 \times S_2 = \{\sigma_1 \cup \sigma_2 \mid \sigma_1 \in S_1, \sigma_2 \in S_2\}\). 
Algorithm~\ref{alg:snr} applies the second pass implicitly by excluding conjuncts on Line~\ref{snr:ln:and}.
This is the standard procedure for computing the prime implicates
of a monotone NNF, but with no subsumption checking which is critical for its complexity.
In the standard procedure, one must ensure that the implicates computed
on Lines~\ref{snr:ln:or} and~\ref{snr:ln:and} are reduced: no implicate \(\sigma_1\) subsumes 
another \(\sigma_2\) (\(\sigma_1 \subseteq \sigma_2\)).\footnote{One can compute the prime
implicates of a monotone NNF by simply converting it to a CNF and then removing subsumed clauses.
Similarly, one can compute the prime implicants of a monotone NNF by converting it to a DNF
and removing subsumed terms. See, e.g.,~\cite{BooleanFunctions}.}
\shrink{
This check takes quadratic time so it can be quite expensive. Moreover, the intermediate
sets of propagated implicates may be much larger than the final set of prime implicates. 
}
Since NNF \(\EXP(\Delta)\) is 
\(\vee\)-decomposable, the disjuncts \(\alpha_1, \ldots, \alpha_n\) on Line~\ref{snr:ln:or}
do not share variables. Hence, if every \(\Call{snr}{\alpha_i}\) is reduced, their Cartesian product is 
reduced. Moreover, due to pruning in the second pass, the
implicates \(\Call{snr}{\alpha_i}\) computed on Line~\ref{snr:ln:and} all have the same length so no subsumption is possible.

\begin{table*}[t]
\centering
\scriptsize
\begin{tabular}{@{}lrrrrrrrrlrlrlrlrl@{}}\toprule
\multicolumn{2}{c}{benchmark} & 
\multicolumn{6}{c}{decision tree/graph properties} &
\multicolumn{2}{c}{\SR} &
\multicolumn{2}{c}{\SSR} &
\multicolumn{2}{c}{\NR} &
\multicolumn{2}{c}{\SNR} 
\\ 
\cmidrule(lr){1-2}
\cmidrule(lr){3-8}
\cmidrule(lr){9-10}
\cmidrule(lr){11-12}
\cmidrule(lr){13-14}
\cmidrule(lr){15-16}
{name} & {examples} & 
{nodes} & {num} & {nom}  & {classes} & {card} & {acc} &
{count} & {time}  &
{count} & {time} &
{count} & {time} &
{count} & {time}  
\\ 
\midrule

adult & 48842 &
726 & 6 & 7 & 2 & 24 & 86.0  &
2.8 & 0.0005 &
1.4 & 0.0007 &
5.6 & 0.0004 &
2.9 & 0.0003
\\
  
compas & 5278 &
55 & 5 & 3 & 2 & 8 & 71.2  &
1.7 & 0.0002 &
1.1 & 0.0003 &
2.9 & 0.0002 &
1.8 & 0.0001
\\
      
fash-mnist  & 70000 &
6681 & 734 & 0 & 10 & 29 & 80.8 &
1851.6 & 3.7925\(^{\mathbf{657}}\) &
7.5 & 0.0348\(^{\mathbf{8}}\) &
104.8 & 0.0123 &
12.4 & 0.0008
\\

gisette & 7000 & 
231 & 111 & 0 & 2 & 3 & 93.9 &
3288.4 & 6.2361\(^{\mathbf{486}}\) &
95.5 & 0.0545 &
31.8 & 0.0013 &
7.4 & 0.0004
\\ 

isolet & 7797 &
645 & 201 & 0 & 26 & 7 & 83.9 &
8.9 & 0.0008 &
1.5 & 0.0026 &
17.7 & 0.0006 &
10.1 & 0.0003
\\              

la1s.wc & 3204 &
457 & 201 & 0 & 6 & 3 & 73.2 &
461.1 & 1.8227\(^{\mathbf{36}}\) &
7.7 & 0.0368 &
42.1 & 0.0038 &
19.3 & 0.0004
\\              

mnist-784 & 70000 &
4365 & 477 & 0 & 10 & 18 & 88.3 &
1833.3 & 4.6191\(^{\mathbf{622}}\) &
5.5 & 0.0235\(^{\mathbf{1}}\) &
103.9 & 0.0099 &
11.5 & 0.0008
\\         

nomao & 34465 &
932 & 61 & 25 & 2 & 17 & 95.3 &
1640.8 & 4.9576\(^{\mathbf{156}}\) &
2.4 & 0.0180 &
38.6 & 0.0018 &
4.3 & 0.0005
\\

ohscal.wc & 11162 &
1761 & 582 & 0 & 10 & 6 & 70.9 &
787.1 & 2.0714\(^{\mathbf{225}}\) &
86.2 & 0.3018\(^{\mathbf{21}}\) &
59.9 & 0.0090 &
23.7 & 0.0006
\\       

spambase & 4601&
189 & 37 & 0 & 2 & 8 & 91.5 & 
35.8 & 0.0029 &
4.6 & 0.0060 &
16.9 & 0.0007 &
4.1 & 0.0003
\\           


andes & --- &
5454 & --- & 24 & 2 & 2 & --- &
78.8 & 0.1660 &
2.6 & 0.0295 &
58.0 & 0.0300 &
4.5 & 0.0055
\\                    
  
emdec6g30 & --- &
4154 &  --- & 30 & 2 & 2 & --- &
11.6 & 0.0362 &
1.8 & 0.0445 &
13.6 & 0.0184 &
2.9 & 0.0048
\\       
  
math-skills & --- &
3693629 & --- & 46 & 2 & 2 & --- &
9.6 & 0.2895\(^{\mathbf{26}}\) &
7.6 & 0.0261 &
9.9 & 0.4271\(^{\mathbf{18}}\) &
3.4 & 0.0290
\\         
  
mooring & --- &
14468 & --- & 22 & 2 & 2 & --- &
145.4 & 2.3991\(^{\mathbf{14}}\) &
20.9 & 3.9951\(^{\mathbf{7}}\) &
216.3 & 0.7678\(^{\mathbf{1}}\) &
16.2 & 0.0189
\\

tcc4e38 & --- &
22508 & --- & 38 & 2 & 2 & --- &
14.4 & 0.1370 &
2.6 & 0.0155 &
10.4 & 0.0632 &
2.0 & 0.0103
\\
          
\bottomrule
\end{tabular}
\caption{Evaluating Algorithms~\ref{alg:snr} \&~\ref{alg:ssr}. Times in secs. First ten entries are decision trees. Last five entries are decision graphs.
\label{tab:dt-dg}}
\end{table*}

\begin{proposition}\label{prop:snr-complexity}
Let \(\Delta\) be a monotone, \(\vee\)-decomposable NNF with \(M\) shortest implicates, \(N\) nodes and \(E\) edges. 
The time complexity of \(\Call{snr}{\Delta}\) in Algorithm~\ref{alg:snr} is \(O(M \cdot E)\)
and its space complexity is \(O(M \cdot N)\).
\end{proposition}

We obtain a tighter complexity if we apply Algorithm~\ref{alg:snr}
to the closed-form complete reasons of decision trees given by Proposition~\ref{prop:cr-cf},
due to the following bound on the number of prime implicates  (a superset of shortest implicates).

\begin{proposition}\label{prop:dt-nr}
For a decision tree, the complete reason for an instance in class \(c\) has \(\leq L\) prime implicates,
where \(L\) is the number of leaves in the tree labeled with a class \(c' \neq c\).
\end{proposition}

The complete reason for a decision tree \(T\) has \(O(|T|)\) nodes and edges,
where \(|T|\) is the decision tree size (see Proposition~\ref{prop:cr-cf}). 
The time and space complexity of Algorithm~\ref{alg:snr} is then
\(O(|T| \cdot L)\) for decision trees. 

\cite{corr/abs-2106-01350} showed that the number of contrastive explanations is linear in
the decision tree size.
Proposition~\ref{prop:dt-nr} tightens this result by providing a more specific bound.
For a decision \(T\) with binary variables and binary classes, \cite{Audemard2021OnTE}
showed that the set of all contrastive explanations can be computed in time polynomial in \(|T|+n\), 
where \(n\) is the number of variables.
Algorithm~\ref{alg:snr} comes with a tighter complexity for the computation of shortest contrastive
explanations for decision trees and applies to multi-class decision trees with discrete features.
Another related complexity result is that counterfactual explanations, as discussed earlier,
can be enumerated with polynomial delay if the classifier satisfies some conditions
as stated in~\cite{kr/AudemardKM20}.

We finally observe that if an NNF is monotone and \(\wedge\)-decomposable, then one
can develop a dual of Algorithm~\ref{alg:snr} for computing the shortest prime implicants of the NNF.


{\bf Shortest Sufficient Reasons}
We next present Algorithm~\ref{alg:ssr} for computing the shortest implicants of monotone, \(\vee\)-decomposable NNFs
which is a hard task. 
For decision trees, the problem of deciding whether there exists a sufficient reason 
of length \(\leq k\) is NP-complete~\cite{nips/BarceloM0S20}. 
Since decision trees have closed-form complete reasons that are monotone and \(\vee\)-decomposable,
computing the shortest implicants for this class of NNFs is hard. 
\cite{Audemard2021OnTE} showed that the number of shortest sufficient reasons for decision trees can be exponential 
and provided an incremental algorithm for computing the shortest sufficient reasons for decision trees 
with binary variables and binary classes, based on a reduction to the \textsc{Partial MaxSat} problem.
Algorithm~\ref{alg:ssr} has a broader scope, does not require a reduction and is based on two key techniques.

The first technique is to compute all unsubsumed implicants of length \(\leq k\), called \(k\)-implicants, starting with \(k=0\).
If no implicants are found, \(k\) is incremented and the process is repeated. 
The second technique relates to computing the \(k\)-implicants of a conjunction \(\alpha \wedge \beta\).
If we have the \(k\)-implicants \(S\) for \(\alpha\) and the \(k\)-implicants \(R\) for \(\beta\),
we can compute the Cartesian product \(S \times R\) and keep unsubsumed implicants of
length \(\leq k\). Algorithm~\ref{alg:ssr} does something more refined.
It first computes the \(k\)-implicants \(S\) for \(\alpha\). For each implicant \(\sigma \in S\),
it then computes and accumulates the \(k'\)-implicants for \(\beta | \sigma\) where \(k' = k - |\sigma|\). 
These techniques control the number of generated \(k\)-implicants at each NNF node (smaller \(k\) leads to fewer \(k\)-implicants). 
Our implemented caching scheme on Lines~\ref{ssr:ln:l-cache} \&~\ref{ssr:ln:i-cache}
exploits the following properties. 
If the \(k\)-implicants for \(\Delta | \sigma\) are \(\{\{\}\}\), then these are also its \(j\)-implicants for all \(j\).
Further, if we cached the \(k\)-implicants for \(\Delta | \sigma\), then we can use them to retrieve its
\(j\)-implicants for any \(j \leq k\) by selecting implicants of length \(\leq j\). 
We empirically evaluate Algorithms~\ref{alg:snr} \&~\ref{alg:ssr} next.

\shrink{
We finally note that one can compute a single prime implicant of a monotone, \(\vee\)-decomposable
NNF in polytime using the greedy procedure in~\cite{corr/abs-2106-01350}. This procedure
starts with an instance that satisfies the NNF and repeatedly drops states from the instance while 
ensuring that the partial instance continues to imply the NNF. The procedure has a polytime complexity 
as long as the implication test can be done efficiently. This  test can be done in linear time
if the NNF is monotone or \(\vee\)-decomposable (Lemma~\ref{lem:D-M} in appendix).  
}


{\bf Empirical Evaluation.}
Table~\ref{tab:dt-dg} depicts an empirical evaluation 
on decision trees learned from OpenML datasets \cite{OpenML2013} and binary decision graphs 
compiled from Bayesian network classifiers~\cite{pgm/ShihCD18,aaai/ShihCD19}. The decision trees were 
learned by {\sc weka}~\cite{weka} using python-weka-wrapper3 available at \url{pypi.org}. 
We used {\sc weka}'s J48 classifier with default settings, which learns pruned C4.5 decision trees
with numeric and nominal features~\cite{Quinlan1993}. 
Each dataset was split using {\sc weka} into training (\(85\%\)) and testing (\(15\%\)) data. 
We aimed for OpenML datasets with more than \(100\) features since many smaller datasets we 
tried were very easy, but we kept a few smaller ones since they are commonly reported on (adult, compas, spambase). 
Some of the learned decision trees had significantly fewer variables than the corresponding datasets (e.g., gisette 
has \(5000\) features but the learned decision tree has \(111\)).
The decision graphs we used are the reportedly largest ones compiled by~\cite{pgm/ShihCD18,aaai/ShihCD19}.
For each decision tree, we computed reasons for decisions on \(1000\) instances sampled from testing
data (or all testing data if smaller than \(1000\)). We tried random instances but they were much easier. 
For each decision graph, we computed complete reasons for \(1000\) random instances (there is no corresponding data). 
The total number of instances for the fifteen benchmarks was \(13963\). 
We did not report the time for computing a complete reason as this
is a closed form with linear size (Equation~\ref{eq:cr-cf}).

We compared four algorithms: 
\SNR\ (Algorithm~\ref{alg:snr}), 
\NR\ (standard algorithm for computing prime implicates of a monotone NNF but with no subsumption 
checking at \(\vee\)-nodes since the input NNF is \(\vee\)-decomposable),\footnote{More precisely,
\NR\ is Algorithm~\ref{alg:snr} with two exceptions. First, the NNF is not pruned on Line~\ref{snr:ln:and} so the union is over all \(\alpha_i\). Second,
subsumption checking is applied after Line~\ref{snr:ln:and} to ensure that all computed implicates are subset-minimal.}
\SSR\ (Algorithm~\ref{alg:ssr}) and
\SR\ (dual of \NR).
Each instance had a timeout of \(60\) seconds.
In Table~\ref{tab:dt-dg}, nodes, num, nom, classes; card and acc stand for
number of nodes, numeric features, nominal features, classes; maximum cardinality of variables 
and accuracy. Count and time are averages over instances that both \SR/\SSR\ (\NR/\SNR) finished.
The bolded exponent of time is the number of instances that timed out (not reported if zero).
The supplementary material contains further statistics such as stdev, mean and max.
We used a Python implementation on a dual Intel(R) Xeon E5-2670 CPUs running at 2.60GHz and 256GB RAM.
As revealed by Table~\ref{tab:dt-dg}, \SSR\ is quite effective. Its average running time is normally in milliseconds,
it timed out on only \(37\) instances and can
be two orders of magnitude faster than \SR\ which timed out on \(2222\) instances.
\SNR\ is also much faster than \NR\ but the latter is also very effective on decision trees (see 
Proposition~\ref{prop:dt-nr}) but timed out on \(19\) decision graph instances.
All algorithms are quite effective on the easier benchmarks. 

\section{Conclusion}
We studied the computation of complete reasons for multi-class classifiers with nominal and numeric features. 
We derived closed forms for the complete reasons of decision trees and graphs in the form of monotone, \(\vee\)-decomposable
NNFs and showed how similar forms can be derived for SDDs. We further established a correspondence 
between the prime implicates of complete reasons and contrastive explanations. We then presented an output polynomial
algorithm for enumerating the shortest implicates (shortest necessary reasons) for complete reasons in the above form.
We also presented a simple algorithm for enumerating the shortest implicants (shortest sufficient reasons) which appears to be
effective based on an empirical evaluation over fifteen datasets. 

\section{Acknowledgements}
This work has been partially supported by NSF grant \#ISS-1910317 and ONR grant \#N00014-18-1-2561.

\bibliography{biblio}


\appendix

\onecolumn

\section{Proofs}

\begin{lemma}\label{lem:D-M}
Conditioning an NNF preserves the properties of \(\vee\)-decomposability, \(\wedge\)-decomposability and monotonicity.
Moreover, if an NNF is \(\wedge\)-decomposable or monotone, then its satisfiability can be decided in linear time.
Finally, if an NNF is \(\vee\)-decomposable or monotone, then its validity can be decided in linear time.
\end{lemma}

\begin{proof}[\bf Proof of Lemma~\ref{lem:D-M}]
The first part about conditioning follows directly from the definitions of conditioning, \(\vee\)-decomposability, 
\(\wedge\)-decomposability and monotonicity.
The satisfiability test distributes over disjunctions. It distributes over a conjunction when the conjuncts do not share variables.
The validity test distributions over conjunctions. It distributes over disjunctions when the disjuncts do not share variables.
A monotone NNF is satisfiable (valid) iff it evaluates to true after replacing each of its literals with \(\top\) (\(\bot\)).
\end{proof}

\begin{lemma}\label{lem:state-or}
For state \(x_i\) of variable \(X\) and positive NNF \(\alpha\), we have \(x_i \vee \alpha\ = x_i \vee \alpha[x_i \To \bot]\) 
where \(\alpha[x_i \To \bot]\) is obtained by replacing every occurrence of \(x_i\) in \(\alpha\) with \(\bot\).
\end{lemma}

\begin{proof}[\bf Proof of Lemma~\ref{lem:state-or}]
We prove this by induction on the structure of positive NNF \(\alpha\).
The base cases for \(\alpha\) are \(\top\), \(\bot\), \(x_i\), \(x_j\) for \(j \neq i\), and \(y_j\) for \(Y \neq X\).
When \(\alpha = x_i\), \(\alpha[x_i \To \bot] = \bot\) and \(x_i \vee x_i = x_i \vee \bot\) so the result holds. 
For all other base cases, \(\alpha[x_i \To \bot] = \alpha\) so the result holds. There are two inductive steps for
\(\alpha = \alpha_1 \wedge \alpha_2\) and \(\alpha = \alpha_1 \vee \alpha_2\). 
By the induction hypothesis, \(x_i \vee \alpha_1[x_i \To \bot] = x_i \vee \alpha_1\) 
and \(x_i \vee \alpha_2[x_i \To \bot] = x_i \vee \alpha_2\).
We have 
\((\alpha_1 \wedge \alpha_2)[x_i \To \bot] = (\alpha_1[x_i \To \bot]) \wedge (\alpha_2[x_i \To \bot])\) 
and hence
\(x_i \vee (\alpha_1 \wedge \alpha_2)[x_i \To \bot] = (x_i \vee \alpha_1) \wedge (x_i \vee \alpha_2) = x_i \vee \alpha\).
We also have 
\((\alpha_1 \vee \alpha_2)[x_i \To \bot] = (\alpha_1[x_i \To \bot]) \vee (\alpha_2[x_i \To \bot])\)
and hence
\(x_i \vee (\alpha_1 \vee \alpha_2)[x_i \To \bot] = (x_i \vee \alpha_1) \vee (x_i \vee \alpha_2) = x_i \vee \alpha\).
\end{proof}

\begin{proof}[\bf Proof of Proposition~\ref{prop:d-quantify-b}]
Follows directly from Definition~\ref{def:d-quantify}.
\end{proof}

\begin{proof}[\bf Proof of Proposition~\ref{prop:distribute-and-or}]
For the first part of the proposition, we have:
\begin{eqnarray*}
\forall x_i \cdot (\alpha \wedge \beta)
& = & (\alpha \wedge \beta)\cd x_i \wedge \bigwedge_{j \neq i} (x_i \vee (\alpha \wedge \beta) \cd x_j) \\
& = & (\alpha \cd x_i) \wedge (\beta\cd x_i) 
\wedge \bigwedge_{j \neq i} (x_i \vee \alpha\cd x_j) \wedge
\bigwedge_{j \neq i} (x_i \vee \beta \cd x_j) \\
& = & (\forall x_i \cdot \alpha) \wedge (\forall x_i \cdot \beta). 
\end{eqnarray*}

To show the second part of the proposition, suppose variable \(X\) does not occur in \(\alpha\).
Then \(\alpha \cd x_j = \alpha\) for all \(j\). Hence, \(\forall x_i \cdot \alpha = \alpha\) by Definition~\ref{def:d-quantify}.
We now have:
\begin{eqnarray*}
\forall x_i \cdot (\alpha \vee \beta)
& = & (\alpha \vee \beta)\cd x_i \wedge \bigwedge_{j \neq i} (x_i \vee (\alpha \vee \beta) \cd x_j) \\
& = & (\alpha \vee (\beta \cd x_i)) \wedge (x_i \vee \alpha \vee \bigwedge_{j \neq i} \beta \cd x_j) \\
& = & \alpha \vee (x_i \wedge (\beta \cd x_i)) \vee ((\beta \cd x_i) \wedge \bigwedge_{j \neq i} \beta \cd x_j) \\
& = & \alpha \vee ((\beta \cd x_i) \wedge \bigwedge_{j \neq i} (x_i \vee \beta \cd x_j)) \\
& = & (\forall x_i \cdot \alpha) \vee (\forall x_i \cdot \beta). \qedhere
\end{eqnarray*}
\end{proof}

\begin{proof}[\bf Proof of Proposition~\ref{prop:d-quantify-compute}]
Follows directly from Propositions~\ref{prop:d-quantify-b} and~\ref{prop:distribute-and-or}.
Note that universal quantification is commutative so we can quantify states in any order. 
\end{proof}

\begin{lemma}\label{lem:independence}
Let \(\alpha\) be a positive NNF and let \(x_i\) be a state of variable \(X\).
If \(x_i\) does not occur in \(\alpha\), then 
\(\alpha \cd x_i \models \alpha \cd x_j\) for all \(j \neq i\)
and \(\forall x_i \cdot \alpha = \alpha \cd x_i\).
If \(x_j\) does not occur in \(\alpha\) for \(j \neq i\), then 
\(\alpha \cd x_j \models \alpha \cd x_i\) for \(j \neq i\), \(\alpha \cd x_j = \alpha \cd x_k\)  for \(j, k \neq i\) 
and \(\forall x_i \cdot \alpha = \alpha\).
\end{lemma}

\begin{proof}[\bf Proof of Lemma~\ref{lem:independence}]
Suppose state \(x_i\) does not occur in NNF  \(\alpha\) and let \(j \neq i\).
Then \(\alpha \cd x_i\) is obtained from \(\alpha\) by replacing every state of variable \(X\) in \(\alpha\) with \(\bot\).
Moreover, \(\alpha \cd x_j\) is obtained from \(\alpha\) by replacing state \(x_j\) with \(\top\) and replacing all other states 
of variable \(X\) in \(\alpha\) with \(\bot\). Hence, \(\alpha \cd x_i \models \alpha \cd x_j\). We now have:
\begin{eqnarray*}
\forall x_i \cdot \alpha 
& = & (\alpha \cd x_i) \wedge \bigwedge_{j \neq i} (x_i \vee \alpha \cd x_j) \\
& = & (\alpha \cd x_i) \wedge (x_i \vee \bigwedge_{j \neq i} \alpha \cd x_j) \\
& = & ((\alpha \cd x_i) \wedge x_i) \vee ((\alpha \cd x_i) \wedge \bigwedge_{j\neq i} \alpha \cd x_j )) \\
& = & ((\alpha \cd x_i) \wedge x_i) \vee (\alpha \cd x_i)\quad \mbox{   since \(\alpha \cd x_i \models \alpha \cd x_j\)} \\
& = & \alpha \cd x_i.
\end{eqnarray*}
Suppose state \(x_j\) does not occur in NNF \(\alpha\) for \(j \neq i\).
Then \(\alpha \cd x_j\) is obtained from \(\alpha\) by replacing state \(x_i\) with \(\bot\).
Hence, \(\alpha \cd x_j = \alpha \cd x_k\) for \(k \neq i\).
Moreover, \(\alpha \cd x_i\) is obtained from \(\alpha\) by replacing state \(x_i\) with \(\top\).
Hence, \(\alpha \cd x_j \models \alpha \cd x_i\). We now have:
\begin{eqnarray*}
\forall x_i \cdot \alpha 
& = & (x_i \wedge (\alpha \cd x_i)) \vee ((\alpha \cd x_i) \wedge \bigwedge_{j\neq i} \alpha \cd x_j )) 
	\quad \mbox{ by previous derivation} \\
& = & (x_i \wedge (\alpha \cd x_i)) \vee (\bigwedge_{j\neq i} \alpha \cd x_j) 
	\quad \mbox{   since \(\bigwedge_{j \neq i} \alpha \cd x_j \models \alpha \cd x_i\)} \\
& = & (x_i \wedge (\alpha \cd x_i)) \vee (x_i \wedge \bigwedge_{j\neq i} \alpha \cd x_j) \vee (\n{x}_i \wedge \bigwedge_{j\neq i} \alpha \cd x_j) \\
& = & (x_i \wedge (\alpha \cd x_i)) \vee (\n{x}_i \wedge \bigwedge_{j\neq i} \alpha \cd x_j)
	\quad \mbox{   since \(x_i \wedge \bigwedge_{j \neq i} \alpha \cd x_j \models x_i \wedge (\alpha \cd x_i)\)} \\
& = & (x_i \wedge (\alpha \cd x_i)) \vee \bigvee_{k \neq i} (x_k \wedge \bigwedge_{j\neq i} \alpha \cd x_j) \\
& = & (x_i \wedge (\alpha \cd x_i)) \vee \bigvee_{k \neq i} (x_k \wedge \alpha \cd x_k) 
	\quad \mbox{ since \(\alpha \cd x_k = \alpha \cd x_j\)} \\
& = & \alpha. \qedhere
\end{eqnarray*}
\end{proof}

\begin{proof}[\bf Proof of Proposition~\ref{prop:distribute-or-ind}]
Suppose \(x_i\) does not occur in either \(\alpha\) or \(\beta\). 
Then \(x_i\) does not occur in \(\alpha \vee \beta\).
By Lemma~\ref{lem:independence}, 
\(\forall x_i \cdot \alpha = \alpha \cd x_i\), 
\(\forall x_i \cdot \beta = \beta \cd x_i\) and
\(\forall x_i \cdot (\alpha \vee \beta) = (\alpha \vee \beta) \cd x_i = (\alpha \cd x_i) \vee (\beta \cd x_i) 
= (\forall x_i \cdot \alpha) \vee (\forall x_i \cdot \beta)\).

Suppose \(x_j\) does not occur in either \(\alpha\) or \(\beta\) for \(j \neq i\).
Then \(x_j\) does not occur in \(\alpha \vee \beta\) for \(j \neq i\). 
By Lemma~\ref{lem:independence}, 
\(\forall x_i \cdot \alpha = \alpha\), 
\(\forall x_i \cdot \beta = \beta\) and
\(\forall x_i \cdot (\alpha \vee \beta) = \alpha \vee \beta = (\forall x_i \cdot \alpha) \vee (\forall x_i \cdot \beta)\). 
\end{proof}

\begin{proof}[\bf Proof of Proposition~\ref{prop:wtop}]
Suppose the states of  \(X\) appear in \(\alpha\) only in disjunctions of the form \(\Sigma_{S'} = \bigvee_{k \in S'} x_k\)
where \(S' \supseteq S\).
Then \(\alpha \cd x_p \models \alpha \cd x_q\) for all \(x_p\) and for all \(x_q \in S\). This follows since 
 \(x_q \in S'\) and hence \(\Sigma_{S'} \cd x_q = \top\). We now have the following equality, which is key for the proof: 
\begin{equation}
\bigwedge_{x_j \not \in S} \alpha \cd x_j = \bigwedge_{x_j \not \in S} \alpha \cd x_j \wedge \bigwedge_{x_j \in S} \alpha \cd x_j 
= \bigwedge_{j} \alpha \cd x_j. \label{eq:wtop}
\end{equation}
The proof will consider two cases: \(x_i \in S\) and \(x_i \not \in S\). Suppose \(x_i \in S\). We then have:
\begin{eqnarray*}
\forall x_i \cdot (\alpha \vee \bigvee_{x_k \in S} x_k) 
& = & ((\alpha \vee \bigvee_{x_k \in S} x_k)\cd x_i) \wedge (x_i \vee \bigwedge_{j \neq i} (\alpha \vee \bigvee_{x_k \in S} x_k) \cd x_j) \\
& = & x_i \vee \bigwedge_{x_j \not \in S} (\alpha \vee \bigvee_{x_k \in S} x_k) \cd x_j \\
& = & x_i \vee \bigwedge_{x_j \not \in S} \alpha \cd x_j \\
& = & x_i \vee \bigwedge_{j} \alpha \cd x_j \quad \mbox{by Equation~\ref{eq:wtop}.}
\end{eqnarray*}
Moreover, we have
\begin{eqnarray*}
\forall x_i \cdot \bigvee_{x_k \in S} x_k
& = & x_i. \\
\forall x_i \cdot \alpha
& = & (\alpha \cd x_i) \wedge (x_i \vee \bigwedge_{j \neq i} \alpha \cd x_j).
\end{eqnarray*}
The proposition holds since
\[
(\forall x_i \cdot \alpha) \vee (\forall x_i \cdot \bigvee_{x_k \in S} x_k)
= x_i \vee ((\alpha \cd x_i) \wedge (x_i \vee \bigwedge_{j \neq i} \alpha \cd x_j)) \\
= x_i \vee \bigwedge_{j} \alpha \cd x_j
\]
Suppose now that \(x_i \not \in S\). We then have:
\begin{eqnarray*}
\forall x_i \cdot (\alpha \vee \bigvee_{x_k \in S} x_k) 
& = & ((\alpha \vee \bigvee_{x_k \in S} x_k)\cd x_i) \wedge (x_i \vee \bigwedge_{j \neq i} (\alpha \vee \bigvee_{x_k \in S} x_k) \cd x_j) \\
& = & (\alpha \cd x_i) \wedge (x_i \vee \bigwedge_{j \neq i, x_j \not \in S} \alpha \cd x_j) \\
& = & (x_i \wedge (\alpha \cd x_i)) \vee \bigwedge_{x_j \not \in S} \alpha \cd x_j \\
& = & (x_i \wedge \alpha) \vee \bigwedge_{j} \alpha \cd x_j \quad \mbox{by Equation~\ref{eq:wtop}.}
\end{eqnarray*}
Moreover, we have
\begin{eqnarray*}
\forall x_i \cdot \bigvee_{x_k \in S} x_k
& = & \bot. \\
\forall x_i \cdot \alpha
& = & (\alpha \cd x_i) \wedge (x_i \vee \bigwedge_{j \neq i} \alpha \cd x_j)
= (x_i \wedge \alpha) \vee \bigwedge_{j} \alpha \cd x_j.
\end{eqnarray*}
The proposition holds since
\[
(\forall x_i \cdot \alpha) \vee (\forall x_i \cdot \bigvee_{x_k \in S} x_k)
= (x_i \wedge \alpha) \vee \bigwedge_{j} \alpha \cd x_j. \qedhere
\]
\end{proof}

\begin{proof}[\bf Proof of Proposition~\ref{prop:dt-nnfs}]
Consider the following NNF, which is the standard translation of a decision graph into an NNF
except that we are switching the \(\top\) and \(\bot\) leaf nodes:
\begin{equation*}
\Lambda^c[T] = 
\left\{
\begin{array}{ll}
\bot, & \mbox{if \(T\) is labeled with class \(c\)} \\
\top, & \mbox{if \(T\) is labeled with class \(c' \neq c\)} \\
\displaystyle\bigvee_j (\Lambda^c[T_j] \wedge \bigvee_{x_i \in S_j} x_i),
& \mbox{\(T\) has edges \(\DE(X,S_j,T_j)\)}
\end{array}
\right.
\end{equation*}
This NNF characterizes the instances of classes \(c' \neq c\).
That is, \(\delta \models \Lambda^c[T]\) iff \(T\) assigns a class \(c' \neq c\) to instance \(\delta\).
NNF \(\Delta^c[T]\) of Definition~\ref{def:dt-nnf} results from negating \(\Lambda^c[T]\) 
using deMorgan's law. 
Hence, \(\Delta^c[T]\) characterizes the instances of class \(c\).
\end{proof}

\begin{proof}[\bf Proof of Proposition~\ref{prop:cr-cf}]
We  first prove this proposition for decision graphs that satisfy the test-once property.
In this case, NNF \(\Delta^c[T]\) of Definition~\ref{def:dt-nnf} will be \(\vee\)-decomposable except for the disjunction
\(\Upsilon_j = \bigvee_{x_i \not \in S_j} x_i\) of each node \(T\) with edges edges \(\DE(X,S_j,T_j)\).
By Proposition~\ref{prop:distribute-and-or},
we can compute \(\forall \delta \cdot \Delta^c[T]\) by simply replacing each such disjunction
\(\Upsilon_j\)  with \(\forall \delta \cdot \Upsilon_j\)
since \(\forall \delta\) distributes over all other conjunctions and disjunctions in the NNF.
Let \(x_l = \delta[X]\).
Then  \(\forall \delta \cdot \Upsilon_j = \forall x_l \cdot \Upsilon_j\) by Lemma~\ref{lem:independence}. 
Since \(\{S_j\}_j\) is a state partition for variable \(X\) (test-once property), we have \(x_l \in \bigcup_j S_j\)
and \(x_l \in S_k\) for some  unique \(k\). By Definition~\ref{def:d-quantify}, 
\(\forall x_l \cdot \Upsilon_j = \bot\) if \(j = k\) and \(\forall x_l \cdot \Upsilon_j = x_l = \delta[X]\) if \(j \neq k\).
Applying these substitutions to NNF \(\Delta^c[T]\) of Definition~\ref{def:dt-nnf} 
yields NNF \(\Gamma^c[T]\) of Equation~\ref{eq:cr-cf}.

Suppose now that the decision graph satisfies only the weak test-once property. There is one place
where the previous argument breaks and another place where it is incomplete. 
Consider the disjunction \(\Delta^c[T_j] \vee (\bigvee_{x_i \not \in S_j} x_i)\) 
in Definition~\ref{def:dt-nnf} and suppose variable \(X\) is tested again at some node in 
graph \(T_j\) (not possible under the test-once property). 
Variable \(X\) will then appear in NNF \(\Delta^c[T_j]\) so we can no longer justify the
the distribution of \(\forall \delta\) over \(\Delta^c[T_j]\) and \((\bigvee_{x_i \not \in S_j} x_i)\)
using \(\vee\)-decomposability. Observe however that every occurrence of variable \(X\) in NNF \(\Delta^c[T_j]\) will
be in disjunctions of the form \(\bigvee_{x_i \not \in R_o} x_i\) where 
\(\DE(X,R_o,T_o)\) is an edge in graph \(T_j\). By the weak test-once property,
\(R_o \subseteq S_j\) and, hence, \(\overline{S}_j \subseteq \overline{R}_o\) 
so \(\forall \delta\) will distribute over  the disjunction \(\Delta^c[T_j] \vee (\bigvee_{x_i \not \in S_j} x_i)\)
by Proposition~\ref{prop:wtop} (and Proposition~\ref{prop:distribute-and-or}).
We now consider the place where the previous argument is incomplete. 
This is when \(x_l = \delta[X]\) and \(x_l \not \in \bigcup_o R_o\) (not possible under the test-once property). In this case,
\(\l_o = \forall x_l \cdot (\bigvee_{x_i \not \in R_o} x_i) = x_l\). However, \(\l_o = \bot\) according to
Equation~\ref{eq:cr-cf}. Using \(\bot\) instead of \(x_l\) preserves equivalence of the complete
reason as we show next (and ensures its \(\vee\)-decomposable). 
Since \(x_l \not \in \bigcup_o R_o\), we have \(x_l \not \in S_j\)
by the weak test-once property. Hence,
\(\forall x_l \cdot (\bigvee_{x_i \not \in S_j} x_i) = x_l\) which leads to
\(\forall \delta \cdot (\Delta^c[T_j] \vee \bigvee_{x_i \not \in S_j} x_i) = (\forall \delta \cdot \Delta^c[T_j]) \vee x_l\).
By Lemma~\ref{lem:state-or},
we can therefore replace the occurrences of \(x_l\) in \(\forall \delta \cdot \Delta^c[T_j]\) with \(\bot\) 
while preserving equivalence.
\end{proof}

\begin{proof}[\bf Proof of Proposition~\ref{prop:dt-cr}]
The formula in Equation~\ref{eq:cr-cf} does not contain negations so it is an NNF.
Literals can only appear in this NNF as \(\l_j\) in the expression \(\bigwedge_{j} (\Gamma^c[T_j] \vee \l_j)\) 
for a graph node with edges \(\DE(X,S_j,T_j)\). 
Since \(\l_j = \bot\) or \(\l_j = \delta[X]\), all literals in the NNF are states in instance \(\delta\)
so the NNF is monotone. To show that the NNF is \(\vee\)-decomposable, suppose variable \(X\)
is tested again at some node in graph \(T_j\) which has edges \(\DE(X,R_o,T_o)\).
By the conditions of Equation~\ref{eq:cr-cf}, if \(\l_j = \delta[X]\), then \(\l_j \not \in S_j\).
By the weak test-once property, \(\l_j \not \in R_o\) for all \(o\). 
By the conditions of Equation~\ref{eq:cr-cf}, \(\l_o = \bot\) for all \(o\) 
so state \(\l_j\) will not appear in the expression \(\bigwedge_{o} (\Gamma^c[T_o] \vee \l_o)\) for node \(T_k\). 
Hence, the NNF is \(\vee\)-decomposable.
\end{proof}

\begin{proof}[\bf Proof of Proposition~\ref{prop:congruent}]
We first observe that \(\forall \delta \cdot \Delta\) can be written as a DNF \(\tau_1 + \ldots + \tau_n\) where
\(\tau_i\) are the prime implicants of \(\Delta\) that satisfy \(\tau_i \subseteq \delta\)~\cite{ecai/DarwicheH20}.
Moreover, \(\delta \models \Delta\) by the definition of complete reason (we universally quantify \(\delta\)
from the class formula satisfied by \(\delta\)). 

\(\then\) Suppose instances \(\delta\) and \(\delta^\star\) are congruent.
Since \(\delta \models \Delta\), we get \(\delta \cap \delta^\star \models \Delta\) by Definition~\ref{def:congruent}.
Since \(\delta \cap \delta^\star\) is an implicant of  \(\Delta\) and a subset of \(\delta\), 
then \(\tau_i \subseteq \delta \cap \delta^\star\) for some \(\tau_i\).
Hence, \(\delta\cap\delta^\star \models \forall \delta \cdot \Delta\) and therefore \(\delta^\star \models \forall \delta \cdot \Delta\).

\(\bthen\) Suppose \(\delta^\star \models \forall \delta \cdot \Delta = \tau_1 + \ldots + \tau_n\).
Hence, \(\delta^\star \models \tau_i\) for some \(\tau_i\). Since \(\tau_i \subseteq \delta\), \(\tau_i \subseteq \delta^\star\)
and \(\tau_i \models \forall \delta \cdot \Delta\), we have \(\tau_i \subseteq \delta \cap \delta^\star \models \forall \delta \cdot \Delta\).
Moreover, \(\delta \cap \delta^\star \models \Delta\) since \(\forall \delta \cdot \Delta \models \Delta\)~\cite{darwiche2021quantifying}.
Hence, instances \(\delta\) and \(\delta^\star\) are congruent by Definition~\ref{def:congruent}.
\end{proof}

\begin{proof}[\bf Proof of Proposition~\ref{prop:subsets}]
The complete reason \(\Gamma = \forall \delta \cdot \Delta\) can be written
as a DNF \(\tau_1 + \ldots + \tau_n\) where \(\tau_i\) are the prime implicants of \(\Delta\) which satisfy 
\(\tau_i \subseteq \delta\)~\cite{ecai/DarwicheH20}. Hence, the prime implicants of \(\Gamma\)
are subsets of instance \(\delta\). We can obtain the prime implicates of \(\Gamma\) 
by converting DNF \(\tau_1 + \ldots + \tau_n\) into a CNF and then removing subsumed clauses (since the DNF is monotone). 
Hence, the prime implicates of \(\Gamma\) are also subsets of instance \(\delta\).
\end{proof}

\begin{proof}[\bf Proof of Proposition~\ref{prop:contrastive}]
Let \(\Gamma\) denote the complete reason \(\forall \delta \cdot \Delta\). We first prove the following lemma.
\vspace{-5mm}
\begin{quote}
\begin{lemma}\label{lem:contrastive}
If \(\gamma = \{s_1,\ldots,s_n\} \subseteq \delta\), then 
\(\Gamma \models s_1+\ldots+s_n\) iff \(\delta\setminus\gamma \not \models \Delta\).
\end{lemma}
\begin{proof}[\bf Proof of Lemma~\ref{lem:contrastive}]
Suppose \(\gamma = \{s_1,\ldots,s_n\}\) is a subset of instance \(\delta\).

\(\then\) Suppose \(\Gamma \models s_1+\ldots+s_n\).
Then \(\n{s}_1,\ldots,\n{s}_n \models \neg \Gamma\) and hence
\(\delta\setminus\gamma,\n{s}_1,\ldots,\n{s}_n \models \neg \Gamma\).
Let \(\delta^\star\) be an instance consistent with term \(\delta\setminus\gamma,\n{s}_1,\ldots,\n{s}_n\).
Then \(\delta^\star \models \neg \Gamma\) and \(\delta\cap\delta^\star = \delta\setminus\gamma\).
Since \(\delta^\star \models \neg \Gamma\), instances \(\delta\) and \(\delta^\star\) are not congruent
by Proposition~\ref{prop:congruent} and hence \(\delta\cap\delta^\star \not \models \Delta\) by 
Definition~\ref{def:congruent}. Hence, \(\delta\setminus\gamma \not \models \Delta\).

\(\bthen\) Suppose \(\delta\setminus\gamma \not \models \Delta\).
We will show \(\Gamma \models s_1 + \ldots + s_n\) by contradiction. 
Suppose \(\Gamma \not \models s_1 + \ldots + s_n\).
Then \(\Gamma, \n{s}_1, \ldots, \n{s}_n\) is consistent so there must exist
an instance \(\delta^\star\) such that \(\delta^\star \models \Gamma, \n{s}_1, \ldots, \n{s}_n\).
Since \(\delta^\star \models \Gamma\), instances \(\delta\) and \(\delta^\star\) are congruent 
by Proposition~\ref{prop:congruent} and hence \(\delta \cap \delta^\star \models \Delta\) by Definition~\ref{def:congruent}.
Since \(\delta \cap \delta^\star \subseteq \delta \setminus \gamma\), then
\(\delta \setminus \gamma \models \Delta\), a contradiction. Hence, \(\Gamma \models s_1 + \ldots + s_n\).
\end{proof}
\end{quote}

\noindent We are now ready to prove the proposition. 

\(\then\) Suppose \(\gamma = \{s_1,\ldots,s_n\}\) is a prime implicate of \(\Gamma\).
By Proposition~\ref{prop:subsets}, \(\gamma \subseteq \delta\), so
\(\gamma\) is a minimal subset of \(\delta\) such that \(\Gamma \models s_1+\ldots+s_n\).
By Lemma~\ref{lem:contrastive}, \(\gamma\) is a minimal subset of \(\delta\) 
such that \(\delta\setminus\gamma \not \models \Delta\) so it is a contrastive explanation by Definition~\ref{def:contrastive}.

\(\bthen\) Suppose \(\gamma = \{s_1,\ldots,s_n\}\) is a contrastive explanation.
By Definition~\ref{def:contrastive},
\(\gamma\) is a minimal subset of instance \(\delta\) such that \(\delta\setminus\gamma \not \models \Delta\).
By Lemma~\ref{lem:contrastive}, 
\(\gamma\) is a minimal subset of \(\delta\) such that \(\Gamma \models s_1 + \ldots + s_n\).
Hence, \(\gamma\)  is a prime implicate of \(\Gamma\) by Proposition~\ref{prop:subsets}.
\end{proof}

\begin{proof}[\bf Proof of Proposition~\ref{prop:iml}]
The cases \(\top\), \(\bot\) and \(x_i\) follow directly from Definition~\ref{def:iml}. 
For the other two cases, we first observe the following.
A monotone NNF can be converted to a CNF which includes
all prime (and shortest) implicates of the NNF as follows:
\(\CNF(\bot) = \{\{\}\}\), \(\CNF(\top) = \{\}\), \(\CNF(x_i) = \{\{x_i\}\}\),
\(\CNF(\alpha \wedge \beta) = \CNF(\alpha) \cup \CNF(\beta)\) and
\(\CNF(\alpha \vee \beta) = \{i \cup j \mid i \in \CNF(\alpha), j \in \CNF(\beta)\}\).
This immediately gives \(\iml(\alpha \wedge \beta) = \min(\iml(\alpha),\iml(\beta))\).
If \(\alpha\) and \(\beta\) share no variables (\(\vee\)-decomposability), 
then \(i \cap j = \emptyset\) for \(i \in \CNF(\alpha)\) and \(j \in \CNF(\beta)\),
which leads to \(\iml(\alpha \vee \beta) = \iml(\alpha) + \iml(\beta)\).
\end{proof}

\begin{proof}[\bf Proof of Proposition~\ref{prop:prune}]
\(\EXP(\Delta)\) is obtained by removing edges from \(\Delta\)
so it must be monotone and \(\vee\)-decomposable given that \(\Delta\) satisfies these properties.
The second part of the proposition can be shown directly by induction on the NNF structure while utilizing the conversion to CNF as
shown in the proof of Proposition~\ref{prop:iml}.
\end{proof}

\begin{proof}[\bf Proof of Proposition~\ref{prop:snr-complexity}]
For each node \(\alpha\) in NNF \(\EXP(\Delta)\), we have \(|\SNR(\alpha)| \leq M\).
This follows because on Line~\ref{snr:ln:or}, we have \(|\SNR(\alpha_1 \vee \ldots \vee \alpha_n)| = \prod_{i=1}^n |\SNR(\alpha_i)|\)
since we are computing a Cartesian product.
Moreover, on Line~\ref{snr:ln:and} with \(I = \{i \mid \iml(\alpha_i)=\iml(\Delta)\}\), we have
\(|\SNR(\bigwedge_{i \in I} \alpha_i)| \geq \max_{i \in I}(|\SNR(\alpha_i)|)\) since we are computing a union.
The Cartesian product computation on Line~\ref{snr:ln:or} takes \(O(|\SNR(\alpha_1 \vee \ldots \vee \alpha_n)|) = O(M)\) time and space.
The union computation on Line~\ref{snr:ln:and} takes \(O(\sum_{i \in I} |\SNR(\alpha_i)|) = O(|I| \cdot M)\) time 
and \(O(|\bigcup_{i \in I} \SNR(\alpha_i)|) = O(M)\) space. Suming these complexities over all nodes
in the NNF, we get a time complexity of \(O(M \cdot E)\) and a space complexity of \(O(M \cdot N)\).
\end{proof}

\begin{proof}[\bf Proof of Proposition~\ref{prop:dt-nr}]
We will show this by induction on the structure of the decision tree.
The complete reason for a decision tree is given in Equation~\ref{eq:cr-cf}, which show next for reference:
\begin{equation*}
\Gamma^c[T] = 
\left\{
\begin{array}{ll}
\top, & \mbox{if \(T\) is labeled with class \(c\)}  \\
\bot, & \mbox{if \(T\) is labeled with class \(c' \neq c\)} \\
\displaystyle
\bigwedge_{j} (\Gamma^c[T_j] \vee \l_j), & 
\mbox{if \(T\) has edges \(\DE(X,S_j,T_j)\)}
\end{array}
\right.
\end{equation*}
We will use \(L[T]\) to denote the number of leaf nodes in tree \(T\) which are labeled with classes 
other than \(c\). We will also use \(\#\![T]\) to denote the number of prime implicates for NNF \(\Gamma^c[T]\).
What we need to show is that \(\#\![T] \leq L[T]\).

The base cases are for leaf nodes of the decision tree. If leaf node \(T\) is labeled with class \(c\), then
\(L[T] = 0\) and the complete reason is \(\top\) which has no prime implicates so \(\#\![T] = 0\) and the proposition holds.
If leaf node \(T\) is labeled with class \(c' \neq c\), then \(L[T]=1\) and the complete reason is
\(\bot\) which has a single prime implicate \(\{\}\) so \(\#\![T] = 1\) and the proposition holds.
Consider now an internal node \(T\) with edges \(\DE(X,S_j,T_j)\) which has the complete
reason \(\bigwedge_{j} (\Gamma^c[T_j] \vee \l_j)\). 
Since \(\l_j\) is either a literal or \(\bot\) (by Proposition~\ref{prop:cr-cf}),
and since \(\l_j\) does not appear in \(\Gamma^c[T_j]\) if it is a literal (by Proposition~\ref{prop:dt-cr}), 
the prime implicates for \(\Gamma^c[T_j] \vee \l_j\) are obtained by
disjoining each prime implicate for \(\Gamma^c[T_j]\) with \(\l_j\). 
Hence, the number of prime implicates for \(\Gamma^c[T_j] \vee \l_j\) equals the number of 
prime implicates for \(\Gamma^c[T_j]\), \(\#\![T_j]\). Moreover, since the complete reason is monotone, 
the prime implicates for \(\bigwedge_{j} (\Gamma^c[T_j] \vee \l_j)\) are obtained by taking the 
union of prime implicates for each \(\Gamma^c[T_j] \vee \l_j\) and removing subsumed implicates
from the result.
Hence, \(\#\![T] \leq \sum_j \#\![T_j]\). By the induction hypothesis, \(\#\![T_j] \leq L[T_j]\).
We now have \(\#\![T] \leq \sum_j L[T_j] = L[T]\) which proves the proposition.
\end{proof}

\end{document}